\documentclass{article} 
\usepackage{iclr2022_conference,times}

\usepackage{amsmath,amsfonts,amsthm,bm}

\usepackage{algorithm}
\usepackage[noend]{algpseudocode}
\usepackage{thmtools}
\usepackage{hyperref}
\usepackage{wrapfig}
\usepackage{booktabs}
\usepackage{graphicx}
\usepackage{subcaption}
\usepackage{url}
\usepackage{enumitem}

\title{Learning Long-Term Reward Redistribution \\ via Randomized Return Decomposition}

\author{Zhizhou Ren$^1$, Ruihan Guo$^2$, Yuan Zhou$^3$, Jian Peng$^{145}$ \\
	$^1$University of Illinois at Urbana-Champaign~ $^2$Shanghai Jiao Tong University\\
	$^3$BIMSA~ $^4$AIR, Tsinghua University~ $^5$HeliXon Limited \\
	\texttt{zhizhour@illinois.edu}, \texttt{guoruihan@sjtu.edu.cn} \\
	\texttt{timzhouyuan@gmail.com}, \texttt{jianpeng@illinois.edu}
}

\iclrfinalcopy 
\begin{document}

\maketitle

\begin{abstract}
	Many practical applications of reinforcement learning require agents to learn from sparse and delayed rewards. It challenges the ability of agents to attribute their actions to future outcomes. In this paper, we consider the problem formulation of episodic reinforcement learning with trajectory feedback. It refers to an extreme delay of reward signals, in which the agent can only obtain one reward signal at the end of each trajectory. A popular paradigm for this problem setting is learning with a designed auxiliary dense reward function, namely proxy reward, instead of sparse environmental signals. Based on this framework, this paper proposes a novel reward redistribution algorithm, randomized return decomposition (RRD), to learn a proxy reward function for episodic reinforcement learning. We establish a surrogate problem by Monte-Carlo sampling that scales up least-squares-based reward redistribution to long-horizon problems. We analyze our surrogate loss function by connection with existing methods in the literature, which illustrates the algorithmic properties of our approach. In experiments, we extensively evaluate our proposed method on a variety of benchmark tasks with episodic rewards and demonstrate substantial improvement over baseline algorithms.
\end{abstract}

\section{Introduction}

Scaling reinforcement learning (RL) algorithms to practical applications has become the focus of numerous recent studies, including resource management \citep{mao2016resource}, industrial control \citep{hein2017benchmark}, drug discovery \citep{popova2018deep}, and recommendation systems \citep{chen2018stabilizing}. One of the challenges in these real-world problems is the sparse and delayed environmental rewards. For example, in the molecular structure design problem, the target molecule property can only be evaluated after completing the whole sequence of modification operations \citep{zhou2019optimization}. The sparsity of environmental feedback would complicate the attribution of rewards on agent actions and therefore can hinder the efficiency of learning \citep{rahmandad2009effects}. In practice, it is a common choice to formulate the RL objective with a meticulously designed reward function instead of the sparse environmental rewards. The design of such a reward function is crucial to the performance of the learned policies. Most standard RL algorithms, such as temporal difference learning and policy gradient methods, prefer dense reward functions that can provide instant feedback for every step of environment transitions. Designing such dense reward functions is not a simple problem even with domain knowledge and human supervision. It has been widely observed in prior works that handcrafted heuristic reward functions may lead to unexpected and undesired behaviors \citep{randlov1998learning, bottou2013counterfactual, andrychowicz2017hindsight}. The agent may find a shortcut solution that formally optimizes the given objective but deviates from the desired policies \citep{dewey2014reinforcement, amodei2016concrete}. The reward designer can hardly anticipate all potential side effects of the designed reward function, which highlights the difficulty of reward engineering.

To avoid the unintended behaviors induced by misspecified reward engineering, a common paradigm is considering the reward design as an online problem within the trial-and-error loop of reinforcement learning \citep{sorg2010reward}. This algorithmic framework contains two components, namely reward modeling and policy optimization. The agent first learns a proxy reward function from the experience data and then optimizes its policy based on the learned per-step rewards. By iterating this procedure and interacting with the environment, the agent is able to continuously refine its reward model so that the learned proxy reward function can better approximate the actual objective given by the environmental feedback. More specifically, this paradigm aims to reshape the sparse and delayed environmental rewards to a dense Markovian reward function while trying to avoid misspecifying the goal of given tasks.

In this paper, we propose a novel reward redistribution algorithm based on a classical mechanism called \textit{return decomposition} \citep{arjona2019rudder}. Our method is built upon the least-squares-based return decomposition \citep{efroni2021reinforcement} whose basic idea is training a regression model that decomposes the trajectory return to the summation of per-step proxy rewards. This paradigm is a promising approach to redistributing sparse environmental feedback. Our proposed algorithm, \textit{randomized return decomposition} (RRD), establish a surrogate optimization of return decomposition to improve the scalability in long-horizon tasks. In this surrogate problem, the reward model is trained to predict the episodic return from a random subsequence of the agent trajectory, i.e., we conduct a structural constraint that the learned proxy rewards can approximately reconstruct environmental trajectory return from a small subset of state-action pairs. This design enables us to conduct return decomposition effectively by mini-batch training. Our analysis shows that our surrogate loss function is an upper bound of the original loss of deterministic return decomposition, which gives a theoretical interpretation of this randomized implementation. We also present how the surrogate gap can be controlled and draw connections to another method called uniform reward redistribution. In experiments, we demonstrate substantial improvement of our proposed approach over baseline algorithms on a suite of MuJoCo benchmark tasks with episodic rewards.

\section{Background}

\subsection{Episodic Reinforcement Learning with Trajectory Feedback}

In standard reinforcement learning settings, the environment model is usually formulated by a \textit{Markov decision process} \citep[MDP;][]{richard1957dynamic}, defined as a tuple $\mathcal{M}=\langle \mathcal{S}, \mathcal{A}, P, R, \mu \rangle$, where $\mathcal{S}$ and $\mathcal{A}$ denote the spaces of environment states and agent actions. $P(s'|s,a)$ and $R(s,a)$ denote the unknown environment transition and reward functions. $\mu$ denotes the initial state distribution. The goal of reinforcement learning is to find a policy $\pi:\mathcal{S}\rightarrow\mathcal{A}$ maximizing cumulative rewards. More specifically, a common objective is maximizing infinite-horizon discounted rewards based on a pre-defined discount factor $\gamma$ as follows:
\begin{align} \label{eq:standard_objective}
	\text{(standard objective)}\qquad J(\pi)=\mathbb{E}\left[\left.\sum_{t=0}^\infty\gamma^tR(s_t,\pi(s_t))~\right|~ s_0\sim\mu,~s_{t+1}\sim P(\cdot\mid s_t,\pi(s_t))\right].
\end{align}

In this paper, we consider the episodic reinforcement learning setting with trajectory feedback, in which the agent can only obtain one reward feedback at the end of each trajectory. Let $\tau$ denote an agent trajectory  that contains all experienced states and behaved actions within an episode. We assume all trajectories terminate in finite steps. The episodic reward function $R_{\text{ep}}(\tau)$ is defined on the trajectory space, which represents the overall performance of trajectory $\tau$. The goal of episodic reinforcement learning is to maximize the expected trajectory return:
\begin{align} \label{eq:episodic_objective}
	\text{(episodic objective)}\quad~ J_{\text{ep}}(\pi)=\mathbb{E}\biggl[R_{\text{ep}}(\tau) ~\biggl|~ s_0\sim \mu,~a_t=\pi(s_t),~\tau=\langle s_0, a_0, s_1, \cdots,s_T\rangle \biggr.\biggr].
\end{align}
In general, the episodic-reward setting is a particular form of \textit{partially observable Markov decision processes} (POMDPs) where the reward function is non-Markovian. The worst case may require the agent to enumerate the entire exponential-size trajectory space for recovering the episodic reward function. In practical problems, the episodic environmental feedback usually has structured representations. A common structural assumption is the existence of an underlying Markovian reward function $\widehat R(s,a)$ that approximates the episodic reward $R_{\text{ep}}(\tau)$ by a sum-form decomposition,
\begin{align} \label{eq:sum-decomposable_episodic_reward}
\text{(sum-decomposable episodic reward)} \qquad\quad R_{\text{ep}}(\tau) \approx \widehat R_{\text{ep}}(\tau) =  \sum_{t=0}^{T-1}\widehat R(s_t,a_t). \qquad\qquad\quad~
\end{align}
This structure is commonly considered by both theoretical \citep{efroni2021reinforcement} and empirical studies \citep{liu2019sequence, raposo2021synthetic} on long-horizon episodic rewards. It models the situations where the agent objective is measured by some metric with additivity properties, e.g., the distance of robot running, the time cost of navigation, or the number of products produced in a time interval.

\subsection{Reward Redistribution}

The goal of \textit{reward redistribution} is constructing a proxy reward function $\widehat R(s_t,a_t)$ that transforms the episodic-reward problem stated in Eq.~\eqref{eq:episodic_objective} to a standard dense-reward setting. By replacing environmental rewards with such a Markovian proxy reward function $\widehat R(s_t,a_t)$, the agent can be trained to optimize the discounted objective in Eq.~\eqref{eq:standard_objective} using any standard RL algorithms. Formally, the proxy rewards $\widehat R(s_t,a_t)$ form a sum-decomposable reward function $\widehat R_{\text{ep}}(\tau) =  \sum_{t=0}^{T-1}\widehat R(s_t,a_t)$ that is expected to have high correlation to the environmental reward $R_{\text{ep}}(\tau)$. Here, we introduce two branches of existing reward redistribution methods, \textit{return decomposition} and \textit{uniform reward redistribution}, which are the most related to our proposed approach. We defer the discussions of other related work to section \ref{sec:related_work}.

\paragraph{Return Decomposition.}
The idea of \textit{return decomposition} is training a reward model that predicts the trajectory return with a given state-action sequence \citep{arjona2019rudder}. In this paper, without further specification, we focus on the least-squares-based implementation of return decomposition \citep{efroni2021reinforcement}. The reward redistribution is given by the learned reward model, i.e., decomposing the environmental episodic reward $R_{\text{ep}}(\tau)$ to a Markovian proxy reward function $\widehat R(s,a)$. In practice, the reward modeling is formulated by optimizing the following loss function:
\begin{align} \label{eq:return_decomposition_loss}
	\mathcal{L}_{\text{RD}}(\theta) = \mathop{\mathbb{E}}_{\tau\sim\mathcal{D}} \left[\biggl(R_{\text{ep}}(\tau)-\sum_{t=0}^{T-1}\widehat R_{\theta}(s_t, a_t)\biggr)^2\right],
\end{align}
where $\widehat R_\theta$ denotes the parameterized proxy reward function, $\theta$ denotes the parameters of the learned reward model, and $\mathcal{D}$ denotes the experience dataset collected by the agent. Assuming the sum-decomposable structure stated in Eq.~\eqref{eq:sum-decomposable_episodic_reward}, $\widehat R_{\theta}(s,a)$ is expected to asymptotically concentrate near the ground-truth underlying rewards $\widehat R(s,a)$ when Eq.~\eqref{eq:return_decomposition_loss} is properly optimized \citep{efroni2021reinforcement}.

One limitation of the least-squares-based return decomposition method specified by Eq.~\eqref{eq:return_decomposition_loss} is its scalability in terms of the computation costs. Note that the trajectory-wise episodic reward is the only environmental supervision for reward modeling. Computing the loss function $\mathcal{L}_{\text{RD}}(\theta)$ with a single episodic reward label requires to enumerate all state-action pairs along the whole trajectory. This computation procedure can be expensive in numerous situations, e.g., when the task horizon $T$ is quite long, or the state space $\mathcal{S}$ is high-dimensional. To address this practical barrier, recent works focus on designing reward redistribution mechanisms that can be easily integrated in complex tasks. We will discuss the implementation subtlety of existing methods in section~\ref{sec:experiments}.

\paragraph{Uniform Reward Redistribution.}
To pursue a simple but effective reward redistribution mechanism, IRCR \citep{gangwani2020learning} considers \textit{uniform reward redistribution} which assumes all state-action pairs equally contribute to the return value. It is designed to redistribute rewards in the absence of any prior structure or information. More specifically, the proxy reward $\widehat R_{\text{IRCR}}(s,a)$ is computed by averaging episodic return values over all experienced trajectories containing $(s,a)$,
\begin{align} \label{eq:uniform_credit_assignment}
	\widehat R_{\text{IRCR}}(s,a) = \mathop{\mathbb{E}}_{\tau\sim\mathcal{D}}\bigl[R_{\text{ep}}(\tau) \mid (s,a)\in\tau\bigr].
\end{align}
In this paper, we will introduce a novel reward redistribution mechanism that bridges between return decomposition and uniform reward redistribution.

\section{Reward Redistribution via Randomized Return Decomposition}

In this section, we introduce our approach, randomized return decomposition (RRD), which sets up a surrogate optimization problem of the least-squares-based return decomposition. The proposed surrogate objective allows us to conduct return decomposition on short subsequences of agent trajectories, which is scalable in long-horizon tasks. We provide analyses to characterize the algorithmic property of our surrogate objective function and discuss connections to existing methods.

\subsection{Randomized Return Decomposition with Monte-Carlo Return Estimation}

One practical barrier to apply least-squares-based return decomposition methods in long-horizon tasks is the computation costs of the regression loss in Eq.~\eqref{eq:return_decomposition_loss}, i.e., it requires to enumerate all state-action pairs within the agent trajectory. To resolve this issue, we consider a randomized method that uses a Monte-Carlo estimator to compute the predicted episodic return $\widehat R_{\text{ep},\theta}(\tau)$ as follows:
\begin{align} \label{eq:Monte-Carlo_estimator}
\underbrace{\widehat R_{\text{ep},\theta}(\tau) = \sum_{t=0}^{T-1}\widehat R_\theta(s_t,a_t)}_{\text{Deterministic Computation}} = \mathop{\mathbb{E}}_{\mathcal{I}\sim\rho_T(\cdot)}\left[\frac{T}{|\mathcal{I}|}\sum_{t\in\mathcal{I}}\widehat R_{\theta}(s_t, a_t)\right] \approx \underbrace{\frac{T}{|\mathcal{I}|}\sum_{t\in\mathcal{I}}\widehat R_{\theta}(s_t, a_t)}_{\text{Monte-Carlo Estimation}}~,
\end{align}
where $\mathcal{I}$ denotes a subset of indices. $\rho_T(\cdot)$ denotes an unbiased sampling distribution where each index $t$ has the same probability to be included in $\mathcal{I}$. In this paper, without further specification, $\rho_T(\cdot)$ is constructed by uniformly sampling $K$ distinct indices.
\begin{align} \label{eq:sampling_distribution}
	\rho_T(\cdot) = \text{Uniform}\left(\bigl\{\mathcal{I}\subseteq \mathbb{Z}_T : |\mathcal{I}|=K\bigr\}\right),
\end{align}
where $K$ is a hyper-parameter. In this sampling distribution, each timestep $t$ has the same probability to be covered by the sampled subsequence $\mathcal{I}\sim\rho_T(\cdot)$ so that it gives an unbiased Monte-Carlo estimation of the episodic summation $\widehat R_{\text{ep},\theta}(\tau)$.

\paragraph{Randomized Return Decomposition.}
Based on the idea of using Monte-Carlo estimation shown in Eq.~\eqref{eq:Monte-Carlo_estimator}, we introduce our approach, \textit{randomized return decomposition} (RRD), to improve the scalability of least-squares-based reward redistribution methods. The objective function of our approach is formulated by the randomized return decomposition loss $\mathcal{L}_{\text{Rand-RD}}(\theta)$ stated in Eq.~\eqref{eq:Rand-RD_loss}, in which the parameterized proxy reward function $\widehat R_\theta$ is trained to predict the episodic return $R_{\text{ep}}(\tau)$ given a random subsequence of the agent trajectory. In other words, we integrate the Monte-Carlo estimator (see Eq.~\eqref{eq:Monte-Carlo_estimator}) into the return decomposition loss to obtain the following surrogate loss function:
\begin{align} \label{eq:Rand-RD_loss}
	\mathcal{L}_{\text{Rand-RD}}(\theta) = \mathop{\mathbb{E}}_{\tau\sim\mathcal{D}} \left[\mathop{\mathbb{E}}_{\mathcal{I}\sim\rho_T(\cdot)}\left[\biggl(R_{\text{ep}}(\tau)-\frac{T}{|\mathcal{I}|}\sum_{t\in\mathcal{I}}\widehat R_{\theta}(s_t, a_t)\biggr)^2\right]\right].
\end{align}
In practice, the loss function $\mathcal{L}_{\text{Rand-RD}}(\theta)$ can be estimated by sampling a mini-batch of trajectory subsequences instead of computing $\widehat R_\theta(s_t,a_t)$ for the whole agent trajectory, and thus the implementation of randomized return decomposition is adaptive and flexible in long-horizon tasks.

\subsection{Analysis of Randomized Return Decomposition} \label{sec:analysis}

The main purpose of our approach is establishing a surrogate loss function to improve the scalability of least-squares-based return decomposition in practice. Our proposed method, randomized return decomposition, is a trade-off between the computation complexity and the estimation error induced by the Monte-Carlo estimator. In this section, we show that our approach is an interpolation between between the return decomposition paradigm and uniform reward redistribution, which can be controlled by the hyper-parameter $K$ used in the sampling distribution (see Eq.~\eqref{eq:sampling_distribution}). We present Theorem \ref{thm:loss_decomposition} as a formal characterization of our proposed surrogate objective function.
\begin{restatable}[Loss Decomposition]{theorem}{LossDecomposition} \label{thm:loss_decomposition}
	The surrogate loss function $\mathcal{L}_{\text{Rand-RD}}(\theta)$ can be decomposed to two terms interpolating between return decomposition and uniform reward redistribution.
	\begin{align}\label{eq:loss_decomposition_raw}
	\mathcal{L}_{\text{Rand-RD}}(\theta) &= \mathcal{L}_{\text{RD}}(\theta) + \mathop{\mathbb{E}}_{\tau\sim\mathcal{D}} \Biggl[~\underbrace{\mathop{\text{Var}}_{\mathcal{I}\sim\rho_T(\cdot)}\left[\frac{T}{|\mathcal{I}|}\sum_{t\in\mathcal{I}}\widehat R_{\theta}(s_t, a_t)\right]}_{\text{variance of the Monte-Carlo estimator}}~\Biggr] \\
	\label{eq:loss_decomposition}
	&= \underbrace{\mathcal{L}_{\text{RD}}(\theta)}_{\text{return decomposition}} + \mathop{\mathbb{E}}_{\tau\sim\mathcal{D}} \Biggl[ T^2\cdot  \underbrace{\mathop{\text{Var}}_{(s_t,a_t)\sim \tau}\left[\widehat R_{\theta}(s_t, a_t)\right]}_{\text{uniform reward redistribution}} \cdot \underbrace{\frac{1}{K}\left(1-\frac{K-1}{T-1}\right)}_{\text{interpolation weight}} ~\Biggr],
	\end{align}
	where $K$ denotes the length of sampled subsequences defined in Eq.~\eqref{eq:sampling_distribution}.
\end{restatable}
The proof of Theorem~\ref{thm:loss_decomposition} is based on the bias-variance decomposition formula of mean squared error \citep{kohavi1996bias}. The detailed proofs are deferred to Appendix~\ref{apx:proof}.

\paragraph{Interpretation as Regularization.}
As presented in Theorem~\ref{thm:loss_decomposition}, the randomized decomposition loss can be decomposed to two terms, the deterministic return decomposition loss $\mathcal{L}_{\text{RD}}(\theta)$ and a variance penalty term (see the second term of Eq.~\eqref{eq:loss_decomposition_raw}). The variance penalty term can be regarded as a regularization that is controlled by hyper-parameter $K$. In practical problems, the objective of return decomposition is ill-posed, since the number of trajectory labels is dramatically less than the number of transition samples. There may exist solutions of reward redistribution that formally optimize the least-squares regression loss but serve little functionality to guide the policy learning. Regarding this issue, the variance penalty in randomized return decomposition is a regularizer for reward modeling. It searches for smooth proxy rewards that has low variance within the trajectory. This regularization effect is similar to the mechanism of uniform reward redistribution \citep{gangwani2020learning}, which achieves state-of-the-art performance in the previous literature. In section~\ref{sec:experiments}, our experiments demonstrate that the variance penalty is crucial to the empirical performance of randomized return decomposition.

\paragraph{A Closer Look at Loss Decomposition.}
In addition to the intuitive interpretation of regularization, we will present a detailed characterization of the loss decomposition shown in Theorem~\ref{thm:loss_decomposition}. We interpret this loss decomposition as below:
\begin{enumerate}
	\item Note that the Monte-Carlo estimator used by randomized return decomposition is an unbiased estimation of the proxy episodic return $\widehat R_{\text{ep},\theta}(\tau)$ (see Eq.~\eqref{eq:Monte-Carlo_estimator}). This unbiased property gives the first component of the loss decomposition, i.e, the original return decomposition loss $\mathcal{L}_{\text{RD}}(\theta)$.
	\item Although the Monte-Carlo estimator is unbiased, its variance would contribute to an additional loss term induced by the mean-square operator, i.e., the second component of loss decomposition presented in Eq.~\eqref{eq:loss_decomposition}. This additional term penalizes the variance of the learned proxy rewards under random sampling. This penalty expresses the same mechanism as uniform reward redistribution \citep{gangwani2020learning} in which the episodic return is uniformly redistributed to the state-action pairs in the trajectory.
\end{enumerate}


Based on the above discussions, we can analyze the algorithmic properties of randomized return decomposition by connecting with previous studies.

\subsubsection{Surrogate Optimization of Return Decomposition}
Randomized return decomposition conducts a surrogate optimization of the actual return decomposition. Note that the variance penalty term in Eq.~\eqref{eq:loss_decomposition} is non-negative, our loss function $\mathcal{L}_{\text{Rand-RD}}(\theta)$ serves an upper bound estimation of the original loss $\mathcal{L}_{\text{RD}}(\theta)$ as the following statement.
\begin{restatable}[Surrogate Upper Bound]{proposition}{SurrogateLoss} \label{prop:surrogate_loss}
	The randomized return decomposition loss $\mathcal{L}_{\text{Rand-RD}}(\theta)$ is an upper bound of the actual return decomposition loss function $\mathcal{L}_{\text{RD}}(\theta)$, i.e., $\mathcal{L}_{\text{Rand-RD}}(\theta)\geq \mathcal{L}_{\text{RD}}(\theta)$.
\end{restatable}
Proposition~\ref{prop:surrogate_loss} suggests that optimizing our surrogate loss $\mathcal{L}_{\text{Rand-RD}}(\theta)$ guarantees to optimize an upper bound of the actual return decomposition loss $\mathcal{L}_{\text{RD}}(\theta)$. According to Theorem \ref{thm:loss_decomposition}, the gap between $\mathcal{L}_{\text{Rand-RD}}(\theta)$ and $\mathcal{L}_{\text{RD}}(\theta)$ refers to the variance of subsequence sampling. The magnitude of this gap can be controlled by the hyper-parameter $K$ that refers to the length of sampled subsequences.
\begin{restatable}[Objective Gap]{proposition}{SurrogateGap}
	Let $\mathcal{L}^{(K)}_{\text{Rand-RD}}(\theta)$ denote the randomized return decomposition loss that samples subsequences with length $K$. The gap between $\mathcal{L}^{(K)}_{\text{Rand-RD}}(\theta)$ and $\mathcal{L}_{\text{RD}}(\theta)$ can be reduced by using larger values of hyper-parameter $K$.
	\begin{align} \label{eq:objective_gap}
		\forall \theta,\quad \mathcal{L}^{(1)}_{\text{Rand-RD}}(\theta) \geq \mathcal{L}^{(2)}_{\text{Rand-RD}}(\theta) \geq \cdots \geq \mathcal{L}^{(T-1)}_{\text{Rand-RD}}(\theta) \geq \mathcal{L}^{(T)}_{\text{Rand-RD}}(\theta) = \mathcal{L}_{\text{RD}}(\theta).
	\end{align}
\end{restatable}
This gap can be eliminated by choosing $K=T$ in the sampling distribution (see Eq.~\eqref{eq:sampling_distribution}) so that our approach degrades to the original deterministic implementation of return decomposition.

\subsubsection{Generalization of Uniform Reward Redistribution}
The reward redistribution mechanism of randomized return decomposition is a generalization of uniform reward redistribution. To serve intuitions, we start the discussions with the simplest case using $K=1$ in subsequence sampling, in which our approach degrades to the uniform reward redistribution as the following statement.
\begin{restatable}[Uniform Reward Redistribution]{proposition}{DegradeToUniformRewardRedistribution}
	Assume all trajectories have the same length and the parameterization space of $\theta$ serves universal representation capacity. The optimal solution $\theta^\star$ of minimizing $\mathcal{L}^{(1)}_{\text{Rand-RD}}(\theta)$ is stated as follows:
	\begin{align} \label{eq:randomized_uniform_redistribution}
		\widehat R_{\theta^\star}(s,a) = \mathop{\mathbb{E}}_{\tau\sim\mathcal{D}}\bigl[R_{\text{ep}}(\tau)/T \mid (s,a)\in\tau\bigr],
	\end{align}
	where $\mathcal{L}^{(1)}_{\text{Rand-RD}}(\theta)$ denotes the randomized return decomposition loss with $K=1$.
\end{restatable}
A minor difference between Eq.~\eqref{eq:randomized_uniform_redistribution} and the proxy reward designed by \citet{gangwani2020learning} (see Eq.~\eqref{eq:uniform_credit_assignment}) is a multiplier scalar $1/T$. \citet{gangwani2020learning} interprets such a proxy reward mechanism as a trajectory-space smoothing process or a non-committal reward redistribution. Our analysis can give a mathematical characterization to illustrate the objective of uniform reward redistribution. As characterized by Theorem~\ref{thm:loss_decomposition}, uniform reward redistribution conducts an additional regularizer to penalize the variance of per-step proxy rewards. In the view of randomized return decomposition, the functionality of this regularizer is requiring the reward model to reconstruct episodic return from each single-step transition.

By using larger values of hyper-parameter $K$, randomized return decomposition is trained to reconstruct the episodic return from a subsequence of agent trajectory instead of the single-step transition used by uniform reward redistribution. This mechanism is a generalization of uniform reward redistribution, in which we equally assign rewards to subsequences generated by uniformly random sampling. It relies on the concentratability of random sampling, i.e., the average of a sequence can be estimated by a small random subset with sub-linear size. The individual contribution of each transition within the subsequence is further attributed by return decomposition. 

\subsection{Practical Implementation of Randomized Return Decomposition}

In Algorithm \ref{alg:framework}, we integrate randomized return decomposition with policy optimization. It follows an iterative paradigm that iterates between the rewarding modeling and policy optimization modules.

\begin{algorithm}[htp]
	\caption{Policy Optimization with Randomized Return Decomposition}
	\label{alg:framework}
	\begin{algorithmic}[1]
		\State Initialize $\mathcal{D}\gets \emptyset$
		\For{$\ell$ = $1,2, \cdots$}
		\State Collect a rollout trajectory $\tau$ using the current policy.
		\State Store trajectory $\tau$ and feedback $R_{\text{ep}}(\tau)$ into the replay buffer $\mathcal{D}\gets \mathcal{D}\cup\{(\tau, R_{\text{ep}}(\tau))\}$.
		\For{$i$ = $1,2, \cdots$}
		\State Sample $M$ trajectories $\{\tau_j\in\mathcal{D}\}_{j=1}^M$ from the replay buffer.
		\State Sample subsequences $\{\mathcal{I}_j\subseteq\mathbb{Z}_{T_j}\}_{j=1}^M$ for these trajectories.
		\State Estimate randomized return decomposition loss $\widehat{\mathcal{L}}_{\text{Rand-RD}}(\theta)$,
		\begin{align} \label{eq:mini-batch_estimation}
		\widehat{\mathcal{L}}_{\text{Rand-RD}}(\theta) = \frac{1}{M}\sum_{j=1}^M \left[\biggl(R_{\text{ep}}(\tau_j)-\frac{T_j}{|\mathcal{I}_j|}\sum_{t\in\mathcal{I}_j}\widehat R_{\theta}(s_{j,t}, a_{j,t})\biggr)^2\right],
		\end{align}
		\qquad\quad where $T_j$ denotes the length of trajectory $\tau_j=\langle s_{j,1}, a_{j,1},\cdots,s_{j,T_j}\rangle$.
		\State Perform a gradient update on the reward model $\widehat R_\theta$,
		\begin{align} \label{eq:mini-batch_gradient_descent} 
		\theta\gets \theta - \alpha \nabla_\theta \widehat{\mathcal{L}}_{\text{Rand-RD}}(\theta),
		\end{align}
		\qquad\quad where $\alpha$ denotes the learning rate.
		\EndFor
		\State Perform policy optimization using the learned proxy reward function $\widehat R_\theta(s,a)$.
		\EndFor
	\end{algorithmic}
\end{algorithm}

As presented in Eq.~\eqref{eq:mini-batch_estimation} and Eq.~\eqref{eq:mini-batch_gradient_descent}, the optimization of our loss function $\mathcal{L}_{\text{Rand-RD}}(\theta)$ can be easily conducted by mini-batch gradient descent. This surrogate loss function only requires computations on short-length subsequences. It provides a scalable implementation for return decomposition that can be generalized to long-horizon tasks with manageable computation costs. In section \ref{sec:experiments}, we will show that this simple implementation can also achieve state-of-the-art performance in comparison to other existing methods.

%
\section{Experiments} \label{sec:experiments}

In this section, we investigate the empirical performance of our proposed methods by conducting experiments on a suite of MuJoCo benchmark tasks with episodic rewards. We compare our approach with several baseline algorithms in the literature and conduct an ablation study on subsequence sampling that is the core component of our algorithm.

\subsection{Performance Evaluation on MuJoCo Benchmark with Episodic Rewards}

\paragraph{Experiment Setting.}
We adopt the same experiment setting as \citet{gangwani2020learning} to compare the performance of our approach with baseline algorithms. The experiment environments is based on the MuJoCo locomotion benchmark tasks created by OpenAI Gym \citep{brockman2016openai}. These tasks are long-horizon with maximum trajectory length $T=1000$, i.e., the task horizon is definitely longer than the batch size used by the standard implementation of mini-batch gradient estimation. We modify the reward function of these environments to set up an episodic-reward setting. Formally, on non-terminal states, the agent will receive a zero signal instead of the per-step dense rewards. The agent can obtain the episodic feedback $R_{\text{ep}}(\tau)$ at the last step of the rollout trajectory, in which $R_{\text{ep}}(\tau)$ is computed by the summation of per-step instant rewards given by the standard setting. We evaluate the performance of our proposed methods with the same configuration of hyper-parameters in all environments. A detailed description of implementation details is included in Appendix~\ref{apx:implementation_details}.

We evaluate two implementations of randomized return decomposition (RRD):
\begin{itemize}[leftmargin=25pt]
	\item \textbf{RRD (ours)} denotes the default implementation of our approach. We train a reward model using randomized return decomposition loss $\mathcal{L}_{\text{Rand-RD}}$, in which we sample subsequences with length $K=64$ in comparison to the task horizon $T=1000$. The reward model $\widehat R_\theta$ is parameterized by a two-layer fully connected network. The policy optimization module is implemented by soft actor-critic \citep[SAC;][]{haarnoja2018soft}. \footnote{The source code of our implementation is available at \url{https://github.com/Stilwell-Git/Randomized-Return-Decomposition}.}
	\item \textbf{RRD-$\mathcal{L}_{\text{RD}}$ (ours)} is an alternative implementation that optimizes $\mathcal{L}_{\text{RD}}$ instead of $\mathcal{L}_{\text{Rand-RD}}$. Note that Theorem~\ref{thm:loss_decomposition} gives a closed-form characterization of the gap between $\mathcal{L}_{\text{Rand-RD}}(\theta)$ and $\mathcal{L}_{\text{RD}}(\theta)$, which is represented by the variance of the learned proxy rewards. By subtracting an unbiased variance estimation from loss function $\mathcal{L}_{\text{Rand-RD}}(\theta)$, we can estimate loss function $\mathcal{L}_{\text{RD}}(\theta)$ by sampling short subsequences. It gives a computationally efficient way to optimize $\mathcal{L}_{\text{RD}}(\theta)$. We include this alternative implementation to reveal the functionality of the regularization given by variance penalty. A detailed description is deferred to Appendix~\ref{apx:alternative_implementation}.
\end{itemize}

We compare with several existing methods for episodic or delayed reward settings:
\begin{itemize}[leftmargin=25pt]
	\item \textbf{IRCR} \citep{gangwani2020learning} is an implementation of uniform reward redistribution. The reward redistribution mechanism of IRCR is non-parametric, in which the proxy reward of a transition is set to be the normalized value of corresponding trajectory return. It is equivalent to use a Monte-Carlo estimator of Eq.~\eqref{eq:uniform_credit_assignment}. Due to the ease of implementation, this method achieves state-of-the-art performance in the literature.
	\item \textbf{RUDDER} \citep{arjona2019rudder} is based on the idea of return decomposition but does not directly optimize $\mathcal{L}_{\text{RD}}(\theta)$. Instead, it trains a return predictor based on trajectory, and the step-wise credit is assigned by the prediction difference between two consecutive states. By using the warm-up technique of LSTM, this transform prevents its training computation costs from depending on the task horizon $T$ so that it is adaptive to long-horizon tasks.
	\item \textbf{GASIL} \citep{guo2018generative}, generative adversarial self-imitation learning, is a generalization of GAIL \citep{ho2016generative}. It formulates an imitation learning framework by imitating  best trajectories in the replay buffer. The proxy rewards are given by a discriminator that is trained to classify the agent and expert trajectories.
	\item \textbf{LIRPG} \citep{zheng2018learning} aims to learn an intrinsic reward function from sparse environment feedback. Its policy is trained to optimized the sum of the extrinsic and intrinsic rewards. The parametric intrinsic reward function is updated by meta-gradients to optimize the actual extrinsic rewards achieved by the policy
\end{itemize}

\begin{figure}[t]
	\centering
	\includegraphics[width=0.98\linewidth]{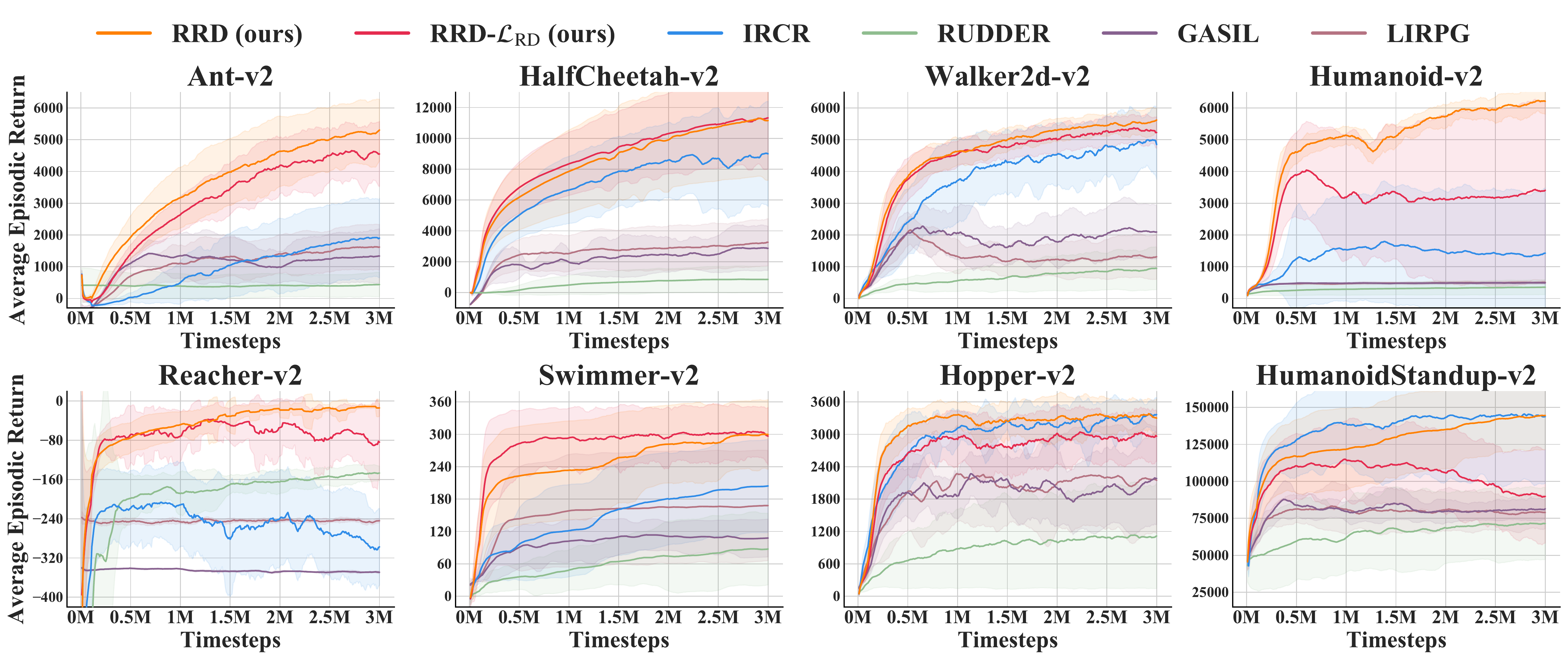}
	\vspace{-0.05in}
	\caption{Learning curves on a suite of MuJoCo benchmark tasks with episodic rewards. All curves for MuJoCo benchmark are plotted from 30 runs with random initializations. The shaded region indicates the standard deviation. To make the comparison more clear, the curves are smoothed by averaging 10 most recent evaluation points. We set up an evaluation point every $10^4$ timesteps.}
	\label{fig:mujoco-experiments}
	\vspace{-0.15in}
\end{figure}

\paragraph{Overall Performance Comparison.}
As presented in Figure~\ref{fig:mujoco-experiments}, randomized return decomposition generally outperforms baseline algorithms. Our approach can achieve higher sample efficiency and produce better policies after convergence. RUDDER is an implementation of return decomposition that represents single-step rewards by the differences between the return predictions of two consecutive states. This implementation maintains high computation efficiency but long-term return prediction is a hard optimization problem and requires on-policy samples. In comparison, RRD is a more scalable and stable implementation which can better integrate with off-policy learning for improving sample efficiency. The uniform reward redistribution considered by IRCR is simple to implement but cannot extract the temporal structure of episodic rewards. Thus the final policy quality produced by RRD is usually better than that of IRCR. GASIL and LIRPG aim to construct auxiliary reward functions that have high correlation to the environmental return. These two methods cannot achieve high sample efficiency since their objectives require on-policy training.

\paragraph{Variance Penalty as Regularization.}
Figure~\ref{fig:mujoco-experiments} also compares two implementations of randomized return decomposition. In most testing environments, RRD optimizing $\mathcal{L}_{\text{Rand-RD}}$ outperforms the unbiased implementation RRD-$\mathcal{L}_{\text{RD}}$. We consider RRD using $\mathcal{L}_{\text{Rand-RD}}$ as our default implementation since it performs better and its objective function is simpler to implement. As discussed in section~\ref{sec:analysis}, the variance penalty conducted by RRD aims to minimize the variance of the Monte-Carlo estimator presented in Eq.~\eqref{eq:Monte-Carlo_estimator}. It serves as a regularization to restrict the solution space of return decomposition, which gives two potential effects: (1) RRD prefers smooth proxy rewards when the expressiveness capacity of reward network over-parameterizes the dataset. (2) The variance of mini-batch gradient estimation can also be reduced when the variance of Monte-Carlo estimator is small. In practice, this regularization would benefit the training stability. As presented in Figure~\ref{fig:mujoco-experiments}, RRD achieves higher sample efficiency than RRD-$\mathcal{L}_{\text{RD}}$ in most testing environments. The quality of the learned policy of RRD is also better than that of RRD-$\mathcal{L}_{\text{RD}}$. It suggests that the regularized reward redistribution can better approximate the actual environmental objective.

\subsection{Ablation Studies}

We conduct an ablation study on the hyper-parameter $K$ that represent the length of subsequences used by randomized return decomposition. As discussed in section~\ref{sec:analysis}, the hyper-parameter $K$ controls the interpolation ratio between return decomposition and uniform reward redistribution. It trades off the accuracy of return reconstruction and variance regularization. In Figure~\ref{fig:ablation-studies}, we evaluate RRD with a set of choices of hyper-parameter $K\in\{1,8,16,32,64,128\}$. The experiment results show that, although the sensitivity of this hyper-parameter depends on the environment, a relatively large value of $K$ generally achieves better performance, since it can better approximate the environmental objective. In this experiment, we ensure all runs use the same input size for mini-batch training, i.e., using larger value of $K$ leads to less number of subsequences in the mini-batch. More specifically, in this ablation study, all algorithm instances estimate the loss function $\mathcal{L}_{\text{Rand-RD}}$ using a mini-batch containing 256 transitions. We consider $K=64$ as our default configuration. As presented in Figure~\ref{fig:ablation-studies}, larger values give marginal improvement in most environments. The benefits of larger values of $K$ can only be observed in HalfCheetah-v2. We note that HalfCheetah-v2 does not have early termination and thus has the longest average horizon among these locomotion tasks. It highlights the trade-off between the weight of regularization and the bias of subsequence estimation.


\begin{figure}[t]
	\centering
	\includegraphics[width=0.98\linewidth]{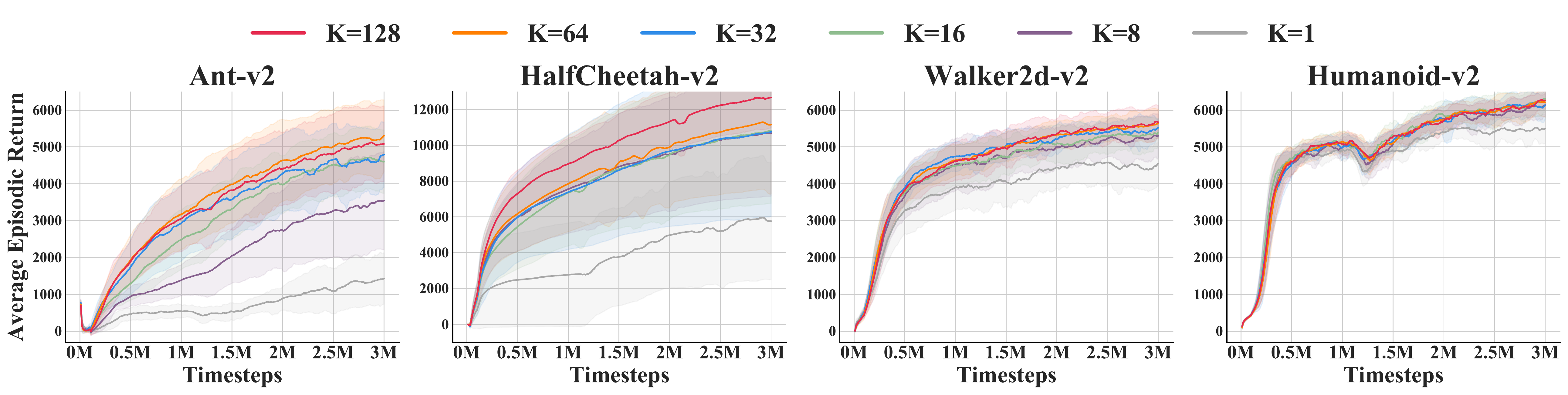}
	\vspace{-0.1in}
	\caption{Learning curves of RRD with different choices of hyper-parameter $K$. The curves with $K=64$ correspond to the default implementation of RRD presented in Figure~\ref{fig:mujoco-experiments}.}
	\label{fig:ablation-studies}
	\vspace{-0.15in}
\end{figure}
\section{Related Work} \label{sec:related_work}

\vspace{-0.08in}
\paragraph{Reward Design.}
The ability of RL agents highly depends on the designs of reward functions. It is widely observed that reward shaping can accelerate learning \citep{mataric1994reward, ng1999, devlin2011empirical, wu2017training, song2019playing}. Many previous works study how to automatically design auxiliary reward functions for efficient reinforcement learning. A famous paradigm, inverse RL \citep{ng2000algorithms, fu2018learning}, considers to recover a reward function from expert demonstrations. Another branch of work is learning an intrinsic reward function that guides the agent to maximize extrinsic objective \citep{sorg2010reward, guo2016deep}. Such an intrinsic reward function can be learned through meta-gradients \citep{zheng2018learning, zheng2020can} or self-imitation \citep{guo2018generative, gangwani2019learning}. A recent work \citep{abel2021expressivity} studies the expressivity of Markov rewards and proposes algorithms to design Markov rewards for three notions of abstract tasks.

\vspace{-0.08in}
\paragraph{Temporal Credit Assignment.}
Another methodology for tackling long-horizon sequential decision problems is assigning credits to emphasize the contribution of each single step over the temporal structure. These methods directly consider the specification of the step values instead of manipulating the reward function. The simplest example is studying how the choice of discount factor $\gamma$ affects the policy learning \citep{petrik2008biasing, jiang2015dependence, fedus2019hyperbolic}. Several previous works consider to extend the $\lambda$-return mechanism \citep{sutton1988learning} to a more generalized credit assignment framework, such as adaptive $\lambda$ \citep{xu2018meta} and pairwise weights \citep{zheng2021pairwise}. RUDDER \citep{arjona2019rudder} proposes a return-equivalent formulation for the credit assignment problem and establish theoretical analyses \citep{holzleitner2021convergence}. Aligned-RUDDER \citep{patil2020align} considers to use expert demonstrations for higher sample efficiency. \citet{harutyunyan2019hindsight} opens up a new family of algorithms, called hindsight credit assignment, that attributes the credits from a backward view. In Appendix~\ref{apx:related_work}, we cover more topics of related work and discuss the connections to the problem focused by this paper.

\vspace{-0.01in}
\section{Conclusion}

\vspace{-0.01in}
In this paper, we propose randomized return decomposition (RRD), a novel reward redistribution algorithm, to tackle the episodic reinforcement learning problem with trajectory feedback. RRD uses a Monte-Carlo estimator to establish a surrogate optimization problem of return decomposition. This surrogate objective implicitly conducts a variance reduction penalty as regularization. We analyze its algorithmic properties by connecting with previous studies in reward redistribution. Our experiments demonstrate that RRD outperforms previous methods in terms of both sample efficiency and policy quality. The basic idea of randomized return decomposition can potentially generalize to other related problems with sum-decomposition structure, such as preference-based reward modeling \citep{christiano2017deep} and multi-agent value decomposition \citep{sunehag2018value}. It is also promising to consider non-linear decomposition as what is explored in multi-agent value factorization \citep{rashid2018qmix}. We leave these investigations as our future work.

\subsubsection*{Acknowledgments}

The authors would like to thank Kefan Dong for insightful discussions. This work is supported by the National Science Foundation under Grant CCF-2006526.

\bibliography{ref}

\begin{thebibliography}{86}
\providecommand{\natexlab}[1]{#1}
\providecommand{\url}[1]{\texttt{#1}}
\expandafter\ifx\csname urlstyle\endcsname\relax
  \providecommand{\doi}[1]{doi: #1}\else
  \providecommand{\doi}{doi: \begingroup \urlstyle{rm}\Url}\fi

\bibitem[Abel et~al.(2021)Abel, Dabney, Harutyunyan, Ho, Littman, Precup, and
  Singh]{abel2021expressivity}
David Abel, Will Dabney, Anna Harutyunyan, Mark~K Ho, Michael Littman, Doina
  Precup, and Satinder Singh.
\newblock On the expressivity of {Markov} reward.
\newblock \emph{Advances in Neural Information Processing Systems}, 34, 2021.

\bibitem[Amodei et~al.(2016)Amodei, Olah, Steinhardt, Christiano, Schulman, and
  Man{\'e}]{amodei2016concrete}
Dario Amodei, Chris Olah, Jacob Steinhardt, Paul Christiano, John Schulman, and
  Dan Man{\'e}.
\newblock Concrete problems in {AI} safety.
\newblock \emph{arXiv preprint arXiv:1606.06565}, 2016.

\bibitem[Amuru \& Buehrer(2014)Amuru and Buehrer]{amuru2014optimal}
Saidhiraj Amuru and R~Michael Buehrer.
\newblock Optimal jamming using delayed learning.
\newblock In \emph{2014 IEEE Military Communications Conference}, pp.\
  1528--1533. IEEE, 2014.

\bibitem[Andrychowicz et~al.(2017)Andrychowicz, Wolski, Ray, Schneider, Fong,
  Welinder, McGrew, Tobin, Abbeel, and Zaremba]{andrychowicz2017hindsight}
Marcin Andrychowicz, Filip Wolski, Alex Ray, Jonas Schneider, Rachel Fong,
  Peter Welinder, Bob McGrew, Josh Tobin, Pieter Abbeel, and Wojciech Zaremba.
\newblock Hindsight experience replay.
\newblock In \emph{Advances in Neural Information Processing Systems}, pp.\
  5048--5058, 2017.

\bibitem[Antos et~al.(2008)Antos, Szepesv{\'a}ri, and Munos]{antos2008learning}
Andr{\'a}s Antos, Csaba Szepesv{\'a}ri, and R{\'e}mi Munos.
\newblock Learning near-optimal policies with {Bellman}-residual minimization
  based fitted policy iteration and a single sample path.
\newblock \emph{Machine Learning}, 71\penalty0 (1):\penalty0 89--129, 2008.

\bibitem[Arjona-Medina et~al.(2019)Arjona-Medina, Gillhofer, Widrich,
  Unterthiner, Brandstetter, and Hochreiter]{arjona2019rudder}
Jose~A Arjona-Medina, Michael Gillhofer, Michael Widrich, Thomas Unterthiner,
  Johannes Brandstetter, and Sepp Hochreiter.
\newblock {RUDDER}: Return decomposition for delayed rewards.
\newblock In \emph{Advances in Neural Information Processing Systems},
  volume~32, 2019.

\bibitem[Bellman(1957)]{richard1957dynamic}
Richard Bellman.
\newblock Dynamic programming.
\newblock \emph{Princeton University Press}, 89:\penalty0 92, 1957.

\bibitem[B{\"o}hmer et~al.(2020)B{\"o}hmer, Kurin, and
  Whiteson]{bohmer2020deep}
Wendelin B{\"o}hmer, Vitaly Kurin, and Shimon Whiteson.
\newblock Deep coordination graphs.
\newblock In \emph{International Conference on Machine Learning}, pp.\
  980--991. PMLR, 2020.

\bibitem[Bottou et~al.(2013)Bottou, Peters, Qui{\~n}onero-Candela, Charles,
  Chickering, Portugaly, Ray, Simard, and Snelson]{bottou2013counterfactual}
L{\'e}on Bottou, Jonas Peters, Joaquin Qui{\~n}onero-Candela, Denis~X Charles,
  D~Max Chickering, Elon Portugaly, Dipankar Ray, Patrice Simard, and
  Ed~Snelson.
\newblock Counterfactual reasoning and learning systems: The example of
  computational advertising.
\newblock \emph{Journal of Machine Learning Research}, 14\penalty0 (11), 2013.

\bibitem[Bouteiller et~al.(2021)Bouteiller, Ramstedt, Beltrame, Pal, and
  Binas]{bouteiller2021reinforcement}
Yann Bouteiller, Simon Ramstedt, Giovanni Beltrame, Christopher Pal, and
  Jonathan Binas.
\newblock Reinforcement learning with random delays.
\newblock In \emph{International Conference on Learning Representations}, 2021.

\bibitem[Brockman et~al.(2016)Brockman, Cheung, Pettersson, Schneider,
  Schulman, Tang, and Zaremba]{brockman2016openai}
Greg Brockman, Vicki Cheung, Ludwig Pettersson, Jonas Schneider, John Schulman,
  Jie Tang, and Wojciech Zaremba.
\newblock {OpenAI} gym.
\newblock \emph{arXiv preprint arXiv:1606.01540}, 2016.

\bibitem[Castro et~al.(2018)Castro, Moitra, Gelada, Kumar, and
  Bellemare]{castro2018dopamine}
Pablo~Samuel Castro, Subhodeep Moitra, Carles Gelada, Saurabh Kumar, and Marc~G
  Bellemare.
\newblock Dopamine: A research framework for deep reinforcement learning.
\newblock \emph{arXiv preprint arXiv:1812.06110}, 2018.

\bibitem[Chatterji et~al.(2021)Chatterji, Pacchiano, Bartlett, and
  Jordan]{chatterji2021theory}
Niladri Chatterji, Aldo Pacchiano, Peter Bartlett, and Michael Jordan.
\newblock On the theory of reinforcement learning with once-per-episode
  feedback.
\newblock \emph{Advances in Neural Information Processing Systems}, 34, 2021.

\bibitem[Chen et~al.(2018)Chen, Yu, Da, Tan, Huang, and
  Tang]{chen2018stabilizing}
Shi-Yong Chen, Yang Yu, Qing Da, Jun Tan, Hai-Kuan Huang, and Hai-Hong Tang.
\newblock Stabilizing reinforcement learning in dynamic environment with
  application to online recommendation.
\newblock In \emph{Proceedings of the 24th ACM SIGKDD International Conference
  on Knowledge Discovery \& Data Mining}, pp.\  1187--1196, 2018.

\bibitem[Chen et~al.(2020)Chen, Li, Umarov, Gao, and Song]{chen2020rna}
Xinshi Chen, Yu~Li, Ramzan Umarov, Xin Gao, and Le~Song.
\newblock Rna secondary structure prediction by learning unrolled algorithms.
\newblock In \emph{International Conference on Learning Representations}, 2020.

\bibitem[Christiano et~al.(2017)Christiano, Leike, Brown, Martic, Legg, and
  Amodei]{christiano2017deep}
Paul~F Christiano, Jan Leike, Tom Brown, Miljan Martic, Shane Legg, and Dario
  Amodei.
\newblock Deep reinforcement learning from human preferences.
\newblock \emph{Advances in neural information processing systems}, 30, 2017.

\bibitem[Clarke(1971)]{clarke1971multipart}
Edward~H Clarke.
\newblock Multipart pricing of public goods.
\newblock \emph{Public choice}, pp.\  17--33, 1971.

\bibitem[Devlin et~al.(2011)Devlin, Kudenko, and
  Grze{\'s}]{devlin2011empirical}
Sam Devlin, Daniel Kudenko, and Marek Grze{\'s}.
\newblock An empirical study of potential-based reward shaping and advice in
  complex, multi-agent systems.
\newblock \emph{Advances in Complex Systems}, 14\penalty0 (02):\penalty0
  251--278, 2011.

\bibitem[Dewey(2014)]{dewey2014reinforcement}
Daniel Dewey.
\newblock Reinforcement learning and the reward engineering principle.
\newblock In \emph{2014 AAAI Spring Symposium Series}, 2014.

\bibitem[Du et~al.(2019)Du, Han, Fang, Dai, Liu, and Tao]{du2019liir}
Yali Du, Lei Han, Meng Fang, Tianhong Dai, Ji~Liu, and Dacheng Tao.
\newblock {LIIR}: learning individual intrinsic reward in multi-agent
  reinforcement learning.
\newblock In \emph{Advances in Neural Information Processing Systems}, pp.\
  4403--4414, 2019.

\bibitem[Efroni et~al.(2021)Efroni, Merlis, and
  Mannor]{efroni2021reinforcement}
Yonathan Efroni, Nadav Merlis, and Shie Mannor.
\newblock Reinforcement learning with trajectory feedback.
\newblock In \emph{Proceedings of the AAAI Conference on Artificial
  Intelligence}, volume~35, pp.\  7288--7295, 2021.

\bibitem[Fedus et~al.(2019)Fedus, Gelada, Bengio, Bellemare, and
  Larochelle]{fedus2019hyperbolic}
William Fedus, Carles Gelada, Yoshua Bengio, Marc~G Bellemare, and Hugo
  Larochelle.
\newblock Hyperbolic discounting and learning over multiple horizons.
\newblock \emph{arXiv preprint arXiv:1902.06865}, 2019.

\bibitem[Fu et~al.(2018)Fu, Luo, and Levine]{fu2018learning}
Justin Fu, Katie Luo, and Sergey Levine.
\newblock Learning robust rewards with adverserial inverse reinforcement
  learning.
\newblock In \emph{International Conference on Learning Representations}, 2018.

\bibitem[Fujimoto et~al.(2018)Fujimoto, Hoof, and
  Meger]{fujimoto2018addressing}
Scott Fujimoto, Herke Hoof, and David Meger.
\newblock Addressing function approximation error in actor-critic methods.
\newblock In \emph{International Conference on Machine Learning}, pp.\
  1587--1596, 2018.

\bibitem[Gangwani et~al.(2019)Gangwani, Liu, and Peng]{gangwani2019learning}
Tanmay Gangwani, Qiang Liu, and Jian Peng.
\newblock Learning self-imitating diverse policies.
\newblock In \emph{International Conference on Learning Representations}, 2019.

\bibitem[Gangwani et~al.(2020)Gangwani, Zhou, and Peng]{gangwani2020learning}
Tanmay Gangwani, Yuan Zhou, and Jian Peng.
\newblock Learning guidance rewards with trajectory-space smoothing.
\newblock In \emph{Advances in Neural Information Processing Systems},
  volume~33, pp.\  822--832, 2020.

\bibitem[Groves(1973)]{groves1973incentives}
Theodore Groves.
\newblock Incentives in teams.
\newblock \emph{Econometrica: Journal of the Econometric Society}, pp.\
  617--631, 1973.

\bibitem[Guo et~al.(2016)Guo, Singh, Lewis, and Lee]{guo2016deep}
Xiaoxiao Guo, Satinder~P Singh, Richard~L Lewis, and Honglak Lee.
\newblock Deep learning for reward design to improve {Monte Carlo} tree search
  in {ATARI} games.
\newblock In \emph{Proceedings of the Twenty-Fifth International Joint
  Conference on Artificial Intelligence}, pp.\  1519–1525, 2016.

\bibitem[Guo et~al.(2018)Guo, Oh, Singh, and Lee]{guo2018generative}
Yijie Guo, Junhyuk Oh, Satinder Singh, and Honglak Lee.
\newblock Generative adversarial self-imitation learning.
\newblock \emph{arXiv preprint arXiv:1812.00950}, 2018.

\bibitem[Haarnoja et~al.(2018{\natexlab{a}})Haarnoja, Zhou, Abbeel, and
  Levine]{haarnoja2018soft}
Tuomas Haarnoja, Aurick Zhou, Pieter Abbeel, and Sergey Levine.
\newblock Soft actor-critic: Off-policy maximum entropy deep reinforcement
  learning with a stochastic actor.
\newblock In \emph{International conference on machine learning}, pp.\
  1861--1870. PMLR, 2018{\natexlab{a}}.

\bibitem[Haarnoja et~al.(2018{\natexlab{b}})Haarnoja, Zhou, Hartikainen,
  Tucker, Ha, Tan, Kumar, Zhu, Gupta, Abbeel, et~al.]{haarnoja2018soft2}
Tuomas Haarnoja, Aurick Zhou, Kristian Hartikainen, George Tucker, Sehoon Ha,
  Jie Tan, Vikash Kumar, Henry Zhu, Abhishek Gupta, Pieter Abbeel, et~al.
\newblock Soft actor-critic algorithms and applications.
\newblock \emph{arXiv preprint arXiv:1812.05905}, 2018{\natexlab{b}}.

\bibitem[Han et~al.(2021)Han, Ren, Wu, Zhou, and Peng]{han2021off}
Beining Han, Zhizhou Ren, Zuofan Wu, Yuan Zhou, and Jian Peng.
\newblock Off-policy reinforcement learning with delayed rewards.
\newblock \emph{arXiv preprint arXiv:2106.11854}, 2021.

\bibitem[Harutyunyan et~al.(2019)Harutyunyan, Dabney, Mesnard, Gheshlaghi~Azar,
  Piot, Heess, van Hasselt, Wayne, Singh, Precup,
  et~al.]{harutyunyan2019hindsight}
Anna Harutyunyan, Will Dabney, Thomas Mesnard, Mohammad Gheshlaghi~Azar, Bilal
  Piot, Nicolas Heess, Hado~P van Hasselt, Gregory Wayne, Satinder Singh, Doina
  Precup, et~al.
\newblock Hindsight credit assignment.
\newblock In \emph{Advances in neural information processing systems},
  volume~32, pp.\  12488--12497, 2019.

\bibitem[Hein et~al.(2017)Hein, Depeweg, Tokic, Udluft, Hentschel, Runkler, and
  Sterzing]{hein2017benchmark}
Daniel Hein, Stefan Depeweg, Michel Tokic, Steffen Udluft, Alexander Hentschel,
  Thomas~A Runkler, and Volkmar Sterzing.
\newblock A benchmark environment motivated by industrial control problems.
\newblock In \emph{2017 IEEE Symposium Series on Computational Intelligence
  (SSCI)}, pp.\  1--8. IEEE, 2017.

\bibitem[H{\'e}liou et~al.(2020)H{\'e}liou, Mertikopoulos, and
  Zhou]{heliou2020gradient}
Am{\'e}lie H{\'e}liou, Panayotis Mertikopoulos, and Zhengyuan Zhou.
\newblock Gradient-free online learning in continuous games with delayed
  rewards.
\newblock In \emph{International Conference on Machine Learning}, pp.\
  4172--4181. PMLR, 2020.

\bibitem[Hester \& Stone(2013)Hester and Stone]{hester2013texplore}
Todd Hester and Peter Stone.
\newblock {TEXPLORE}: real-time sample-efficient reinforcement learning for
  robots.
\newblock \emph{Machine learning}, 90\penalty0 (3):\penalty0 385--429, 2013.

\bibitem[Ho \& Ermon(2016)Ho and Ermon]{ho2016generative}
Jonathan Ho and Stefano Ermon.
\newblock Generative adversarial imitation learning.
\newblock In \emph{Advances in neural information processing systems},
  volume~29, pp.\  4565--4573, 2016.

\bibitem[Holzleitner et~al.(2021)Holzleitner, Gruber, Arjona-Medina,
  Brandstetter, and Hochreiter]{holzleitner2021convergence}
Markus Holzleitner, Lukas Gruber, Jos{\'e} Arjona-Medina, Johannes
  Brandstetter, and Sepp Hochreiter.
\newblock Convergence proof for actor-critic methods applied to {PPO} and
  {RUDDER}.
\newblock In \emph{Transactions on Large-Scale Data-and Knowledge-Centered
  Systems XLVIII}, pp.\  105--130. Springer, 2021.

\bibitem[Jiang et~al.(2015)Jiang, Kulesza, Singh, and
  Lewis]{jiang2015dependence}
Nan Jiang, Alex Kulesza, Satinder Singh, and Richard Lewis.
\newblock The dependence of effective planning horizon on model accuracy.
\newblock In \emph{Proceedings of the 2015 International Conference on
  Autonomous Agents and Multiagent Systems}, pp.\  1181--1189. Citeseer, 2015.

\bibitem[Katsikopoulos \& Engelbrecht(2003)Katsikopoulos and
  Engelbrecht]{katsikopoulos2003markov}
Konstantinos~V Katsikopoulos and Sascha~E Engelbrecht.
\newblock {Markov} decision processes with delays and asynchronous cost
  collection.
\newblock \emph{IEEE transactions on automatic control}, 48\penalty0
  (4):\penalty0 568--574, 2003.

\bibitem[Kingma \& Ba(2015)Kingma and Ba]{kingma2015adam}
Diederik~P Kingma and Jimmy Ba.
\newblock Adam: A method for stochastic optimization.
\newblock In \emph{International Conference on Learning Representations}, 2015.

\bibitem[Kohavi \& Wolpert(1996)Kohavi and Wolpert]{kohavi1996bias}
Ron Kohavi and David Wolpert.
\newblock Bias plus variance decomposition for zero-one loss functions.
\newblock In \emph{Proceedings of the Thirteenth International Conference on
  International Conference on Machine Learning}, pp.\  275--283, 1996.

\bibitem[Lee et~al.(2021)Lee, Smith, and Abbeel]{lee2021pebble}
Kimin Lee, Laura Smith, and Pieter Abbeel.
\newblock {PEBBLE}: Feedback-efficient interactive reinforcement learning via
  relabeling experience and unsupervised pre-training.
\newblock In \emph{International Conference on Machine Learning}, 2021.

\bibitem[Lei et~al.(2020)Lei, Tan, Zheng, Liu, Zhang, and Shen]{lei2020deep}
Lei Lei, Yue Tan, Kan Zheng, Shiwen Liu, Kuan Zhang, and Xuemin Shen.
\newblock Deep reinforcement learning for autonomous internet of things: Model,
  applications and challenges.
\newblock \emph{IEEE Communications Surveys \& Tutorials}, 22\penalty0
  (3):\penalty0 1722--1760, 2020.

\bibitem[Lillicrap et~al.(2016)Lillicrap, Hunt, Pritzel, Heess, Erez, Tassa,
  Silver, and Wierstra]{lillicrap2016continuous}
Timothy~P Lillicrap, Jonathan~J Hunt, Alexander Pritzel, Nicolas Heess, Tom
  Erez, Yuval Tassa, David Silver, and Daan Wierstra.
\newblock Continuous control with deep reinforcement learning.
\newblock In \emph{International Conference on Learning Representations}, 2016.

\bibitem[Liu et~al.(2019)Liu, Luo, Zhong, Chen, Liu, and Peng]{liu2019sequence}
Yang Liu, Yunan Luo, Yuanyi Zhong, Xi~Chen, Qiang Liu, and Jian Peng.
\newblock Sequence modeling of temporal credit assignment for episodic
  reinforcement learning.
\newblock \emph{arXiv preprint arXiv:1905.13420}, 2019.

\bibitem[Mania et~al.(2018)Mania, Guy, and Recht]{mania2018simple}
Horia Mania, Aurelia Guy, and Benjamin Recht.
\newblock Simple random search provides a competitive approach to reinforcement
  learning.
\newblock \emph{arXiv preprint arXiv:1803.07055}, 2018.

\bibitem[Mao et~al.(2016)Mao, Alizadeh, Menache, and Kandula]{mao2016resource}
Hongzi Mao, Mohammad Alizadeh, Ishai Menache, and Srikanth Kandula.
\newblock Resource management with deep reinforcement learning.
\newblock In \emph{Proceedings of the 15th ACM workshop on hot topics in
  networks}, pp.\  50--56, 2016.

\bibitem[Mataric(1994)]{mataric1994reward}
Maja~J Mataric.
\newblock Reward functions for accelerated learning.
\newblock \emph{Machine learning}, pp.\  181--189, 1994.

\bibitem[Mnih et~al.(2015)Mnih, Kavukcuoglu, Silver, Rusu, Veness, Bellemare,
  Graves, Riedmiller, Fidjeland, Ostrovski, et~al.]{mnih2015human}
Volodymyr Mnih, Koray Kavukcuoglu, David Silver, Andrei~A Rusu, Joel Veness,
  Marc~G Bellemare, Alex Graves, Martin Riedmiller, Andreas~K Fidjeland, Georg
  Ostrovski, et~al.
\newblock Human-level control through deep reinforcement learning.
\newblock \emph{Nature}, 518\penalty0 (7540):\penalty0 529--533, 2015.

\bibitem[Myerson(1981)]{myerson1981optimal}
Roger~B Myerson.
\newblock Optimal auction design.
\newblock \emph{Mathematics of operations research}, 6\penalty0 (1):\penalty0
  58--73, 1981.

\bibitem[Nath et~al.(2021)Nath, Baranwal, and Khadilkar]{nath2021revisiting}
Somjit Nath, Mayank Baranwal, and Harshad Khadilkar.
\newblock Revisiting state augmentation methods for reinforcement learning with
  stochastic delays.
\newblock In \emph{Proceedings of the 30th ACM International Conference on
  Information \& Knowledge Management}, pp.\  1346--1355, 2021.

\bibitem[Ng \& Russell(2000)Ng and Russell]{ng2000algorithms}
Andrew~Y Ng and Stuart Russell.
\newblock Algorithms for inverse reinforcement learning.
\newblock In \emph{Proceedings of the Seventeenth International Conference on
  Machine Learning}. Citeseer, 2000.

\bibitem[Ng et~al.(1999)Ng, Harada, and Russell]{ng1999}
Andrew~Y Ng, Daishi Harada, and Stuart~J Russell.
\newblock Policy invariance under reward transformations: Theory and
  application to reward shaping.
\newblock In \emph{Proceedings of the Sixteenth International Conference on
  Machine Learning}, pp.\  278--287. Morgan Kaufmann Publishers Inc., 1999.

\bibitem[Nguyen et~al.(2018)Nguyen, Kumar, and Lau]{nguyen2018credit}
Duc~Thien Nguyen, Akshat Kumar, and Hoong~Chuin Lau.
\newblock Credit assignment for collective multiagent {RL} with global rewards.
\newblock In \emph{Advances in Neural Information Processing Systems},
  volume~31, pp.\  8102--8113, 2018.

\bibitem[Nilsson et~al.(1998)Nilsson, Bernhardsson, and
  Wittenmark]{nilsson1998stochastic}
Johan Nilsson, Bo~Bernhardsson, and Bj{\"o}rn Wittenmark.
\newblock Stochastic analysis and control of real-time systems with random time
  delays.
\newblock \emph{Automatica}, 34\penalty0 (1):\penalty0 57--64, 1998.

\bibitem[Patil et~al.(2020)Patil, Hofmarcher, Dinu, Dorfer, Blies,
  Brandstetter, Arjona-Medina, and Hochreiter]{patil2020align}
Vihang~P Patil, Markus Hofmarcher, Marius-Constantin Dinu, Matthias Dorfer,
  Patrick~M Blies, Johannes Brandstetter, Jose~A Arjona-Medina, and Sepp
  Hochreiter.
\newblock {Align-RUDDER}: Learning from few demonstrations by reward
  redistribution.
\newblock \emph{arXiv preprint arXiv:2009.14108}, 2020.

\bibitem[Petrik \& Scherrer(2008)Petrik and Scherrer]{petrik2008biasing}
Marek Petrik and Bruno Scherrer.
\newblock Biasing approximate dynamic programming with a lower discount factor.
\newblock \emph{Advances in neural information processing systems},
  21:\penalty0 1265--1272, 2008.

\bibitem[Popova et~al.(2018)Popova, Isayev, and Tropsha]{popova2018deep}
Mariya Popova, Olexandr Isayev, and Alexander Tropsha.
\newblock Deep reinforcement learning for de novo drug design.
\newblock \emph{Science advances}, 4\penalty0 (7):\penalty0 eaap7885, 2018.

\bibitem[Rahmandad et~al.(2009)Rahmandad, Repenning, and
  Sterman]{rahmandad2009effects}
Hazhir Rahmandad, Nelson Repenning, and John Sterman.
\newblock Effects of feedback delay on learning.
\newblock \emph{System Dynamics Review}, 25\penalty0 (4):\penalty0 309--338,
  2009.

\bibitem[Randl{\o}v \& Alstr{\o}m(1998)Randl{\o}v and
  Alstr{\o}m]{randlov1998learning}
Jette Randl{\o}v and Preben Alstr{\o}m.
\newblock Learning to drive a bicycle using reinforcement learning and shaping.
\newblock In \emph{International Conference on Machine Learning}, pp.\
  463--471. Citeseer, 1998.

\bibitem[Raposo et~al.(2021)Raposo, Ritter, Santoro, Wayne, Weber, Botvinick,
  van Hasselt, and Song]{raposo2021synthetic}
David Raposo, Sam Ritter, Adam Santoro, Greg Wayne, Theophane Weber, Matt
  Botvinick, Hado van Hasselt, and Francis Song.
\newblock Synthetic returns for long-term credit assignment.
\newblock \emph{arXiv preprint arXiv:2102.12425}, 2021.

\bibitem[Rashid et~al.(2018)Rashid, Samvelyan, Schroeder, Farquhar, Foerster,
  and Whiteson]{rashid2018qmix}
Tabish Rashid, Mikayel Samvelyan, Christian Schroeder, Gregory Farquhar, Jakob
  Foerster, and Shimon Whiteson.
\newblock {QMIX}: Monotonic value function factorisation for deep multi-agent
  reinforcement learning.
\newblock In \emph{International Conference on Machine Learning}, pp.\
  4295--4304. PMLR, 2018.

\bibitem[Schuitema et~al.(2010)Schuitema, Bu{\c{s}}oniu, Babu{\v{s}}ka, and
  Jonker]{schuitema2010control}
Erik Schuitema, Lucian Bu{\c{s}}oniu, Robert Babu{\v{s}}ka, and Pieter Jonker.
\newblock Control delay in reinforcement learning for real-time dynamic
  systems: a memoryless approach.
\newblock In \emph{2010 IEEE/RSJ International Conference on Intelligent Robots
  and Systems}, pp.\  3226--3231. IEEE, 2010.

\bibitem[Singh(2003)]{singh2003advanced}
Sarjinder Singh.
\newblock \emph{Advanced Sampling Theory With Applications: How Michael""
  Selected"" Amy}, volume~2.
\newblock Springer Science \& Business Media, 2003.

\bibitem[Son et~al.(2019)Son, Kim, Kang, Hostallero, and Yi]{son2019qtran}
Kyunghwan Son, Daewoo Kim, Wan~Ju Kang, David~Earl Hostallero, and Yung Yi.
\newblock {QTRAN}: Learning to factorize with transformation for cooperative
  multi-agent reinforcement learning.
\newblock In \emph{International Conference on Machine Learning}, pp.\
  5887--5896. PMLR, 2019.

\bibitem[Song et~al.(2019)Song, Weng, Su, Yan, Zou, and Zhu]{song2019playing}
Shihong Song, Jiayi Weng, Hang Su, Dong Yan, Haosheng Zou, and Jun Zhu.
\newblock Playing {FPS} games with environment-aware hierarchical reinforcement
  learning.
\newblock In \emph{Proceedings of the Twenty-Eighth International Joint
  Conference on Artificial Intelligence}, pp.\  3475--3482, 2019.

\bibitem[Sorg et~al.(2010)Sorg, Lewis, and Singh]{sorg2010reward}
Jonathan Sorg, Richard~L Lewis, and Satinder Singh.
\newblock Reward design via online gradient ascent.
\newblock In \emph{Advances in Neural Information Processing Systems},
  volume~23, pp.\  2190--2198, 2010.

\bibitem[Sunehag et~al.(2018)Sunehag, Lever, Gruslys, Czarnecki, Zambaldi,
  Jaderberg, Lanctot, Sonnerat, Leibo, Tuyls, et~al.]{sunehag2018value}
Peter Sunehag, Guy Lever, Audrunas Gruslys, Wojciech~Marian Czarnecki, Vinicius
  Zambaldi, Max Jaderberg, Marc Lanctot, Nicolas Sonnerat, Joel~Z Leibo, Karl
  Tuyls, et~al.
\newblock Value-decomposition networks for cooperative multi-agent learning
  based on team reward.
\newblock In \emph{Proceedings of the 17th International Conference on
  Autonomous Agents and MultiAgent Systems}, pp.\  2085--2087, 2018.

\bibitem[Sutton(1988)]{sutton1988learning}
Richard~S Sutton.
\newblock Learning to predict by the methods of temporal differences.
\newblock \emph{Machine learning}, 3\penalty0 (1):\penalty0 9--44, 1988.

\bibitem[Tang et~al.(2021)Tang, Ho, and Liu]{tang2021bandit}
Wei Tang, Chien-Ju Ho, and Yang Liu.
\newblock Bandit learning with delayed impact of actions.
\newblock \emph{Advances in Neural Information Processing Systems}, 34, 2021.

\bibitem[Tavakoli et~al.(2021)Tavakoli, Fatemi, and
  Kormushev]{tavakoli2021learning}
Arash Tavakoli, Mehdi Fatemi, and Petar Kormushev.
\newblock Learning to represent action values as a hypergraph on the action
  vertices.
\newblock In \emph{International Conference on Learning Representations}, 2021.

\bibitem[Vickrey(1961)]{vickrey1961counterspeculation}
William Vickrey.
\newblock Counterspeculation, auctions, and competitive sealed tenders.
\newblock \emph{The Journal of finance}, 16\penalty0 (1):\penalty0 8--37, 1961.

\bibitem[Walsh et~al.(2009)Walsh, Nouri, Li, and Littman]{walsh2009learning}
Thomas~J Walsh, Ali Nouri, Lihong Li, and Michael~L Littman.
\newblock Learning and planning in environments with delayed feedback.
\newblock \emph{Autonomous Agents and Multi-Agent Systems}, 18\penalty0
  (1):\penalty0 83--105, 2009.

\bibitem[Wang et~al.(2021{\natexlab{a}})Wang, Ren, Han, Ye, and
  Zhang]{wang2021towards}
Jianhao Wang, Zhizhou Ren, Beining Han, Jianing Ye, and Chongjie Zhang.
\newblock Towards understanding cooperative multi-agent q-learning with value
  factorization.
\newblock In \emph{Thirty-Fifth Conference on Neural Information Processing
  Systems}, 2021{\natexlab{a}}.

\bibitem[Wang et~al.(2021{\natexlab{b}})Wang, Ren, Liu, Yu, and
  Zhang]{wang2021qplex}
Jianhao Wang, Zhizhou Ren, Terry Liu, Yang Yu, and Chongjie Zhang.
\newblock {QPLEX}: Duplex dueling multi-agent q-learning.
\newblock In \emph{International Conference on Learning Representations},
  2021{\natexlab{b}}.

\bibitem[Wang et~al.(2020)Wang, Zhang, Kim, and Gu]{wang2020shapley}
Jianhong Wang, Yuan Zhang, Tae-Kyun Kim, and Yunjie Gu.
\newblock {Shapley} {Q}-value: a local reward approach to solve global reward
  games.
\newblock In \emph{Proceedings of the AAAI Conference on Artificial
  Intelligence}, volume~34, pp.\  7285--7292, 2020.

\bibitem[Wirth et~al.(2016)Wirth, F{\"u}rnkranz, and Neumann]{wirth2016model}
Christian Wirth, Johannes F{\"u}rnkranz, and Gerhard Neumann.
\newblock Model-free preference-based reinforcement learning.
\newblock In \emph{Thirtieth AAAI Conference on Artificial Intelligence}, 2016.

\bibitem[Wu \& Tian(2017)Wu and Tian]{wu2017training}
Yuxin Wu and Yuandong Tian.
\newblock Training agent for first-person shooter game with actor-critic
  curriculum learning.
\newblock In \emph{International Conference on Learning Representations}, 2017.

\bibitem[Xu et~al.(2018)Xu, van Hasselt, and Silver]{xu2018meta}
Zhongwen Xu, Hado~P van Hasselt, and David Silver.
\newblock Meta-gradient reinforcement learning.
\newblock In \emph{Advances in Neural Information Processing Systems},
  volume~31, pp.\  2396--2407, 2018.

\bibitem[Zheng et~al.(2018)Zheng, Oh, and Singh]{zheng2018learning}
Zeyu Zheng, Junhyuk Oh, and Satinder Singh.
\newblock On learning intrinsic rewards for policy gradient methods.
\newblock In \emph{Advances in Neural Information Processing Systems},
  volume~31, pp.\  4644--4654, 2018.

\bibitem[Zheng et~al.(2020)Zheng, Oh, Hessel, Xu, Kroiss, Van~Hasselt, Silver,
  and Singh]{zheng2020can}
Zeyu Zheng, Junhyuk Oh, Matteo Hessel, Zhongwen Xu, Manuel Kroiss, Hado
  Van~Hasselt, David Silver, and Satinder Singh.
\newblock What can learned intrinsic rewards capture?
\newblock In \emph{International Conference on Machine Learning}, pp.\
  11436--11446. PMLR, 2020.

\bibitem[Zheng et~al.(2021)Zheng, Vuorio, Lewis, and Singh]{zheng2021pairwise}
Zeyu Zheng, Risto Vuorio, Richard Lewis, and Satinder Singh.
\newblock Pairwise weights for temporal credit assignment.
\newblock \emph{arXiv preprint arXiv:2102.04999}, 2021.

\bibitem[Zhou et~al.(2018)Zhou, Mertikopoulos, Bambos, Glynn, Ye, Li, and
  Fei-Fei]{zhou2018distributed}
Zhengyuan Zhou, Panayotis Mertikopoulos, Nicholas Bambos, Peter Glynn, Yinyu
  Ye, Li-Jia Li, and Li~Fei-Fei.
\newblock Distributed asynchronous optimization with unbounded delays: How slow
  can you go?
\newblock In \emph{International Conference on Machine Learning}, pp.\
  5970--5979. PMLR, 2018.

\bibitem[Zhou et~al.(2019{\natexlab{a}})Zhou, Xu, and
  Blanchet]{zhou2019learning}
Zhengyuan Zhou, Renyuan Xu, and Jose Blanchet.
\newblock Learning in generalized linear contextual bandits with stochastic
  delays.
\newblock In \emph{Advances in Neural Information Processing Systems},
  volume~32, pp.\  5197--5208, 2019{\natexlab{a}}.

\bibitem[Zhou et~al.(2019{\natexlab{b}})Zhou, Kearnes, Li, Zare, and
  Riley]{zhou2019optimization}
Zhenpeng Zhou, Steven Kearnes, Li~Li, Richard~N Zare, and Patrick Riley.
\newblock Optimization of molecules via deep reinforcement learning.
\newblock \emph{Scientific reports}, 9\penalty0 (1):\penalty0 1--10,
  2019{\natexlab{b}}.

\end{thebibliography}
\bibliographystyle{iclr2022_conference}

\clearpage

\appendix

\section{Omitted Proofs} \label{apx:proof}

\LossDecomposition*
	
\begin{proof}
	First, we note that random sampling serves an unbiased estimation. i.e.,
	\begin{align*} 
		\mathop{\mathbb{E}}_{\mathcal{I}\sim\rho_T(\cdot)}\left[\frac{T}{|\mathcal{I}|}\sum_{t\in\mathcal{I}}\widehat R_{\theta}(s_t, a_t)\right] = \sum_{t=0}^{T-1} \widehat R_{\theta}(s_t,a_t)=\widehat R_{\text{ep},\theta}(\tau).
	\end{align*}
	
	We can decompose our loss function $\mathcal{L}_{\text{Rand-RD}}(\theta)$ as follows:
	\begin{align*}
		\mathcal{L}_{\text{Rand-RD}}(\theta) &= \mathop{\mathbb{E}}_{\tau\sim\mathcal{D}} \left[\mathop{\mathbb{E}}_{\mathcal{I}\sim\rho_T(\cdot)}\left[\biggl(R_{\text{ep}}(\tau)-\frac{T}{|\mathcal{I}|}\sum_{t\in\mathcal{I}}\widehat R_{\theta}(s_t, a_t)\biggr)^2\right]\right] \\
		&= \mathop{\mathbb{E}}_{\tau\sim\mathcal{D}} \left[\mathop{\mathbb{E}}_{\mathcal{I}\sim\rho_T(\cdot)}\left[\biggl(R_{\text{ep}}(\tau)-\widehat R_{\text{ep},\theta}(\tau)+\widehat R_{\text{ep},\theta}(\tau)-\frac{T}{|\mathcal{I}|}\sum_{t\in\mathcal{I}}\widehat R_{\theta}(s_t, a_t)\biggr)^2\right]\right] \\
		&= \underbrace{\mathop{\mathbb{E}}_{\tau\sim\mathcal{D}} \left[\mathop{\mathbb{E}}_{\mathcal{I}\sim\rho_T(\cdot)}\left[\biggl(R_{\text{ep}}(\tau)-\widehat R_{\text{ep},\theta}(\tau)\biggr)^2\right]\right]}_{=\mathcal{L}_{\text{RD}}(\theta)} \\
		&\quad + \mathop{\mathbb{E}}_{\tau\sim\mathcal{D}} \Biggl[2\bigl(R_{\text{ep}}(\tau)-\widehat R_{\text{ep},\theta}(\tau)\bigr)\cdot \underbrace{\mathop{\mathbb{E}}_{\mathcal{I}\sim\rho_T(\cdot)}\left[\biggl(\widehat R_{\text{ep},\theta}(\tau)-\frac{T}{|\mathcal{I}|}\sum_{t\in\mathcal{I}}\widehat R_{\theta}(s_t, a_t)\biggr)\right]}_{=0}~\Biggr] \\
		&\quad + \mathop{\mathbb{E}}_{\tau\sim\mathcal{D}} \Biggl[~\underbrace{\mathop{\mathbb{E}}_{\mathcal{I}\sim\rho_T(\cdot)}\left[\biggl(\widehat R_{\text{ep},\theta}(\tau)-\frac{T}{|\mathcal{I}|}\sum_{t\in\mathcal{I}}\widehat R_{\theta}(s_t, a_t)\biggr)^2\right]}_{=\mathop{\text{Var}}\left[(T/|\mathcal{I}|)\cdot \sum_{t\in\mathcal{I}}\widehat R_{\theta}(s_t, a_t)\right]}~\Biggr] \\
		&= \mathcal{L}_{\text{RD}}(\theta) + \mathop{\mathbb{E}}_{\tau\sim\mathcal{D}} \Biggl[\mathop{\text{Var}}_{\mathcal{I}\sim\rho_T(\cdot)}\left[\frac{T}{|\mathcal{I}|}\sum_{t\in\mathcal{I}}\widehat R_{\theta}(s_t, a_t)\right]\Biggr].
	\end{align*}
	
	Note our sampling distribution defined in Eq.~\eqref{eq:sampling_distribution} refers to ``\textit{sampling without replacement}'' \citep{singh2003advanced} whose variance can be further decomposed as follows:
	\begin{align*}
		\mathcal{L}_{\text{Rand-RD}}(\theta)
		&= \mathcal{L}_{\text{RD}}(\theta) + \mathop{\mathbb{E}}_{\tau\sim\mathcal{D}} \Biggl[\mathop{\text{Var}}_{\mathcal{I}\sim\rho_T(\cdot)}\left[\frac{T}{|\mathcal{I}|}\sum_{t\in\mathcal{I}}\widehat R_{\theta}(s_t, a_t)\right]\Biggr] \\
		&= \mathcal{L}_{\text{RD}}(\theta) + \mathop{\mathbb{E}}_{\tau\sim\mathcal{D}} \Biggl[ T^2\cdot  \mathop{\text{Var}}_{(s_t,a_t)\sim \tau}\left[\widehat R_{\theta}(s_t, a_t)\right] \cdot \frac{1}{K}\left(1-\frac{K-1}{T-1}\right)\Biggr].
	\end{align*}
\end{proof}

The proof of Theorem~\ref{thm:loss_decomposition} follows a particular form of bias-variance decomposition formula \citep{kohavi1996bias}. Similar decomposition form can also be found in other works in the literature of reinforcement learning \citep{antos2008learning}.

\SurrogateLoss*

\begin{proof}
	Note that the second term of Eq.~\eqref{eq:loss_decomposition_raw} in Theorem \ref{thm:loss_decomposition} expresses the variance of a Monte-Carlo estimator which is clearly non-negative. It directly gives $\mathcal{L}_{\text{Rand-RD}}(\theta)\geq\mathcal{L}_{\text{RD}}(\theta)$.
\end{proof}

An alternative proof of Proposition \ref{prop:surrogate_loss} can be directly given by Jensen's inequality.

\SurrogateGap*

\begin{proof}
	In Eq.~\eqref{eq:loss_decomposition} of Theorem~\ref{thm:loss_decomposition}, the last term $\frac{1}{K}\left(1-\frac{K-1}{T-1}\right)$ monotonically decreases as the hyper-parameter $K$ increases. When $K=T$, this coefficient is equal to zero. It derives Eq.~\eqref{eq:objective_gap} in the given statement.
\end{proof}

\DegradeToUniformRewardRedistribution*

\begin{proof}
	Note that we assume all trajectories have the same length. The optimal solution of this least-squares problem is given by
	\begin{align*}
		\widehat R_{\theta^\star}(s,a) &= \min_{r\in\mathbb{R}} \mathop{\mathbb{E}}_{\tau\sim\mathcal{D}}\left[\left.\left(R_{\text{ep}}(\tau)-T\cdot r\right)^2\right|(s,a)\in\tau\right] \\
		&= \min_{r\in\mathbb{R}} \mathop{\mathbb{E}}_{\tau\sim\mathcal{D}}\left[\frac{1}{T^2}\left.\left(R_{\text{ep}}(\tau)/T- r\right)^2\right|(s,a)\in\tau\right] \\
		&= \min_{r\in\mathbb{R}} \frac{1}{T^2}\mathop{\mathbb{E}}_{\tau\sim\mathcal{D}}\left[\left.\left(R_{\text{ep}}(\tau)/T- r\right)^2\right|(s,a)\in\tau\right] \\
		&= \min_{r\in\mathbb{R}} \mathop{\mathbb{E}}_{\tau\sim\mathcal{D}}\left[\left.\left(R_{\text{ep}}(\tau)/T- r\right)^2\right|(s,a)\in\tau\right] \\
		&= \mathop{\mathbb{E}}_{\tau\sim\mathcal{D}}\bigl[R_{\text{ep}}(\tau)/T \mid (s,a)\in\tau\bigr],
	\end{align*}
	which depends on the trajectory distribution in dataset $\mathcal{D}$.
\end{proof}

If we relax the assumption that all trajectories have the same length, the solution of the above least-squares problem would be a weighted expectation as follows:
\begin{align*}
	\widehat R_{\theta^\star}(s,a) &= \min_{r\in\mathbb{R}} \mathop{\mathbb{E}}_{\tau\sim\mathcal{D}}\left[\left.\left(R_{\text{ep}}(\tau)-T_\tau\cdot r\right)^2\right|(s,a)\in\tau\right] \\
	&= \min_{r\in\mathbb{R}} \mathop{\mathbb{E}}_{\tau\sim\mathcal{D}}\left[T_\tau^2\cdot \left.\left(R_{\text{ep}}(\tau)/T_\tau-r\right)^2\right|(s,a)\in\tau\right] \\
	&= \frac{\sum_{\tau\in\mathcal{D}:(s,a)\in\tau} T_\tau \cdot R_{\text{ep}}(\tau)}{\sum_{\tau\in\mathcal{D}:(s,a)\in\tau} T_\tau^2},
\end{align*}
where $T_\tau$ denotes the length of trajectory $\tau$. This solution can still be interpreted as a uniform reward redistribution, in which the dataset distribution is prioritized by the trajectory length.
\section{Experiment Settings and Implementation Details} \label{apx:implementation_details}

\subsection{MuJoCo Benchmark with Episodic Rewards}

\paragraph{MuJoCo Benchmark with Episodic Rewards.}
We adopt the same experiment setting as \citet{gangwani2020learning} and compare our approach with baselines in a suite of MuJoCo locomotion benchmark tasks with episodic rewards. This experiment setting is commonly used in the literature \citep{mania2018simple, guo2018generative, liu2019sequence, arjona2019rudder, gangwani2019learning, gangwani2020learning}. The environment simulator is based on OpenAI Gym \citep{brockman2016openai}. These tasks are long-horizon with maximum trajectory length $T=1000$. We modify the reward function of these environments to set up an episodic-reward setting. Formally, on non-terminal states, the agent will receive a zero signal instead of the per-step dense rewards. The agent can obtain the episodic feedback $R_{\text{ep}}(\tau)$ at the last step of the rollout trajectory, in which $R_{\text{ep}}(\tau)$ is computed by the summation of per-step instant rewards given by the standard setting.

\paragraph{Hyper-Parameter Configuration For MuJoCo Experiments.}
In MuJoCo experiments, the policy optimization module of RRD is implemented based on soft actor-critic \citep[SAC;][]{haarnoja2018soft}. We evaluate the performance of our proposed methods with the same configuration of hyper-parameters in all environments. The hyper-parameters of the back-end SAC follow the official technical report \citep{haarnoja2018soft2}. We summarize our default configuration of hyper-parameters as the following table:

\begin{table}[h]
	\centering
	\begin{tabular}{c c}
		\toprule
		Hyper-Parameter & Default Configuration \\
		\midrule
		discount factor $\gamma$ & 0.99 \\
		\midrule
		\texttt{\#} hidden layers (all networks) & 2 \\
		\texttt{\#} neurons per layer & 256 \\
		activation & ReLU \\
		optimizer (all losses) & Adam \citep{kingma2015adam} \\
		learning rate & $3\cdot 10^{-4}$ \\
		\midrule
		initial temperature $\alpha_{\text{init}}$ & 1.0 \\
		target entropy & $-\dim(\mathcal{A})$ \\
		Polyak-averaging coefficient & 0.005 \\
		\texttt{\#} gradient steps per environment step & 1 \\
		\texttt{\#} gradient steps per target update & 1 \\
		\midrule
		\texttt{\#} transitions in replay buffer & $10^6$ \\
		\texttt{\#} transitions in mini-batch for training SAC & 256 \\
		\texttt{\#} transitions in mini-batch for training $\widehat R_\theta$ & 256 \\
		\texttt{\#} transitions per subsequence ($K$) & 64 \\
		\texttt{\#} subsequences in mini-batch for training $\widehat R_\theta$ & 4 \\
		\bottomrule
	\end{tabular}
	\caption{The hyper-parameter configuration of RRD in MuJoCo experiments.}
	\label{table:hyper-parameter-mujoco}
\end{table}

In addition to SAC, we also provide the implementations upon DDPG \citep{lillicrap2016continuous} and TD3 \citep{fujimoto2018addressing} in our \href{https://github.com/Stilwell-Git/Randomized-Return-Decomposition}{Github repository}.

\subsection{Atari Benchmark with Episodic Rewards}

\paragraph{Atari Benchmark with Episodic Rewards.}
In addition, we conduct experiments in a suite of Atari games with episodic rewards. The environment simulator is based on OpenAI Gym \citep{brockman2016openai}. Following the standard Atari pre-processing proposed by \citet{mnih2015human}, we rescale each RGB frame to an $84\times 84$ luminance map, and the observation is constructed as a stack of 4 recent luminance maps. We modify the reward function of these environments to set up an episodic-reward setting. Formally, on non-terminal states, the agent will receive a zero signal instead of the per-step dense rewards. The agent can obtain the episodic feedback $R_{\text{ep}}(\tau)$ at the last step of the rollout trajectory, in which $R_{\text{ep}}(\tau)$ is computed by the summation of per-step instant rewards given by the standard setting.

\paragraph{Hyper-Parameter Configuration For Atari Experiments.}
In Atari experiments, the policy optimization module of RRD is implemented based on deep Q-network \citep[DQN;][]{mnih2015human}. We evaluate the performance of our proposed methods with the same configuration of hyper-parameters in all environments. The hyper-parameters of the back-end DQN follow the technical report \citep{castro2018dopamine}. We summarize our default configuration of hyper-parameters as the following table:

\begin{table}[h]
	\centering
	\begin{tabular}{c c}
		\toprule
		Hyper-Parameter & Default Configuration \\
		\midrule
		discount factor $\gamma$ & 0.99 \\
		\texttt{\#} stacked frames in agent observation & 4 \\
		\texttt{\#} \texttt{noop} actions while starting a new episode & 30 \\
		\midrule
		network architecture & DQN \citep{mnih2015human} \\
		optimizer for Q-values & Adam \citep{kingma2015adam} \\
		learning rate for Q-values & $6.25\cdot 10^{-5}$ \\
		\midrule
		optimizer for $\widehat R_\theta$ & Adam \citep{kingma2015adam} \\
		learning rate for $\widehat R_\theta$ & $3\cdot 10^{-4}$ \\
		\midrule
		exploration strategy & $\epsilon$-greedy \\
		$\epsilon$ decaying range - start value & 1.0 \\
		$\epsilon$ decaying range - end value & 0.01 \\
		\texttt{\#} timesteps for $\epsilon$ decaying schedule & 250000 \\
		\texttt{\#} gradient steps per environment step & 0.25 \\
		\texttt{\#} gradient steps per target update & 8000 \\
		\midrule
		\texttt{\#} transitions in replay buffer & $10^6$ \\
		\texttt{\#} transitions in mini-batch for training DQN & 32 \\
		\texttt{\#} transitions in mini-batch for training $\widehat R_\theta$ & 32 \\
		\texttt{\#} transitions per subsequence ($K$) & 32 \\
		\texttt{\#} subsequences in mini-batch for training $\widehat R_\theta$ & 1 \\
		\bottomrule
	\end{tabular}
	\caption{The hyper-parameter configuration of RRD in Atari experiments.}
	\label{table:hyper-parameter-atari}
\end{table}

\subsection{An Alternative Implementation of Randomized Return Decomposition} \label{apx:alternative_implementation}

Recall that the major practical barrier of the least-squares-based return decomposition method specified by $\mathcal{L}_{\text{RD}}(\theta)$ is its scalability in terms of the computation costs. The trajectory-wise episodic reward is the only environmental supervision for reward modeling. Computing the loss function $\mathcal{L}_{\text{RD}}(\theta)$ with a single episodic reward label requires to enumerate all state-action pairs along the whole trajectory.

Theorem~\ref{thm:loss_decomposition} motivates an unbiased implementation of randomized return decomposition that optimizes $\mathcal{L}_{\text{RD}}(\theta)$ instead of $\mathcal{L}_{\text{Rand-RD}}(\theta)$. By rearranging the terms of Eq.~\eqref{eq:loss_decomposition}, we can obtain the difference between $\mathcal{L}_{\text{Rand-RD}}(\theta)$ and $\mathcal{L}_{\text{RD}}(\theta)$ as follows:
\begin{align*}
	\mathcal{L}_{\text{Rand-RD}}(\theta) - \mathcal{L}_{\text{RD}}(\theta) = \mathop{\text{Var}}_{\mathcal{I}\sim\rho_T(\cdot)}\left[\frac{T}{|\mathcal{I}|}\sum_{t\in\mathcal{I}}\widehat R_{\theta}(s_t, a_t)\right].
\end{align*}

Note our sampling distribution $\rho_T(\cdot)$ defined in Eq.~\eqref{eq:sampling_distribution} refers to ``\textit{sampling without replacement}'' \citep{singh2003advanced} whose variance can be estimated as follows:
\begin{align*}
\mathop{\text{Var}}_{\mathcal{I}\sim\rho_T(\cdot)}\left[\frac{T}{|\mathcal{I}|}\sum_{t\in\mathcal{I}}\widehat R_{\theta}(s_t, a_t)\right]
&= T^2\cdot \mathop{\mathbb{E}}_{\mathcal{I}\sim\rho_T(\cdot)} \left[\frac{T-K}{T}\cdot \frac{\sum_{t\in\mathcal{I}}\left(\widehat R_\theta(s_t,a_t)- \bar R_\theta(\mathcal{I};\tau)\right)^2}{K(K-1)}\right],
\end{align*}
where $\bar R_\theta(\mathcal{I};\tau) = \frac{1}{|\mathcal{I}|}\sum_{t\in\mathcal{I}}\widehat R_\theta(s_t,a_t)$. Thus we can obtain an unbiased estimation of this variance penalty by sampling a subsequence $\mathcal{I}$. By subtracting this estimation from ${\widehat {\mathcal{L}}}_{\text{Rand-RD}}(\theta)$, we can obtain an unbiased estimation of $\mathcal{L}_{\text{RD}}(\theta)$. More specifically, we can use the following sample-based loss function to substitute Eq.~\eqref{eq:mini-batch_estimation} in implementation:
\begin{align*}
	\widehat{\mathcal{L}}_{\text{RD}}(\theta) &= \quad \underbrace{\frac{1}{M}\sum_{j=1}^M \left[\biggl(R_{\text{ep}}(\tau_j)-\frac{T_j}{|\mathcal{I}_j|}\sum_{t\in\mathcal{I}_j}\widehat R_{\theta}(s_{j,t}, a_{j,t})\biggr)^2\right]}_{\widehat{\mathcal{L}}_{\text{Rand-RD}}(\theta)} \\
	&\quad  - \frac{1}{M}\sum_{j=1}^M\left[\frac{T_j(T_j-K)}{|\mathcal{I}_j|(|\mathcal{I}_j|-1)}\cdot \sum_{t\in\mathcal{I}_j}\biggl(\widehat R_\theta(s_{j,t},a_{j,t})- \frac{1}{|\mathcal{I}_j|}\sum_{t\in\mathcal{I}_j}\widehat R_\theta(s_{j,t},a_{j,t})\biggr)^2\right].
\end{align*}

The above loss function can be optimized through the same mini-batch training paradigm as what is presented in Algorithm~\ref{alg:framework}.
\section{Experiments on Atari Benchmark with Episodic Rewards}


\begin{figure}[h]
	\centering
	\includegraphics[width=0.98\linewidth]{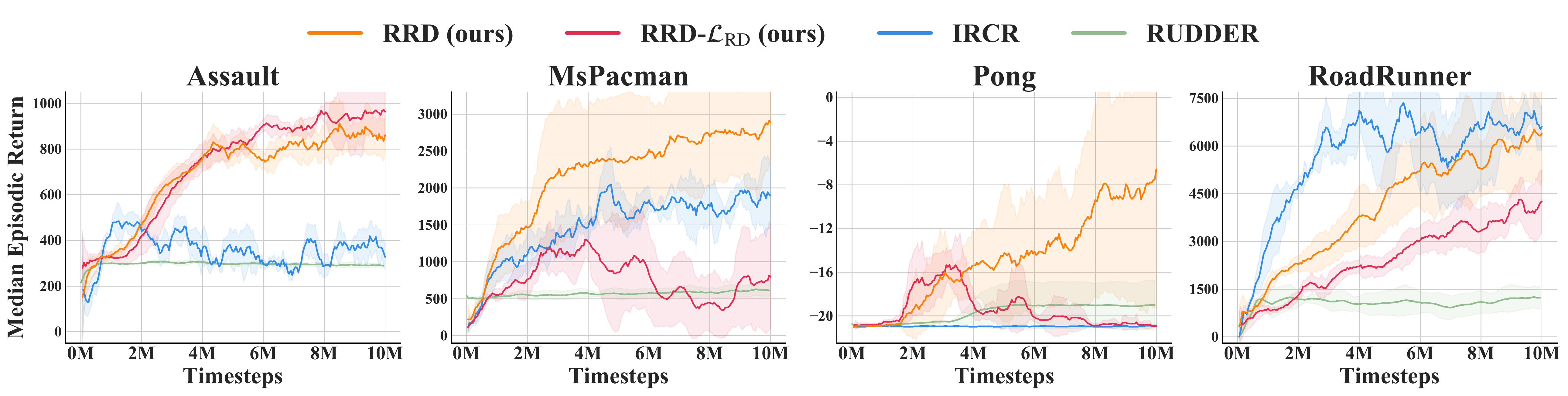}
	\caption{Learning curves on a suite of Atari benchmark tasks with episodic rewards. These curves are plotted from 5 runs with random initializations. The shaded region indicates the standard deviation. To make the comparison more clear, the curves are smoothed by averaging 10 most recent evaluation points. We set up an evaluation point every $5\cdot 10^4$ timesteps.}
	\label{fig:atari-experiments}
\end{figure}

Note that our proposed method does not restricts its usage to continuous control problems. It can also be integrated in DQN-based algorithms to solve problems with discrete-action space. We evaluate the performance of our method built upon DQN in several famous Atari games. The reward redistribution problem in these tasks is more challenging than that in MuJoCo locomotion benchmark since the task horizon of Atari is much longer. For example, the maximum task horizon in Pong can exceed 20000 steps in a single trajectory. This setting highlights the scalability advantage of our method, i.e., the objective of RRD can be optimized by sampling short subsequences whose computation cost is manageable. The experiment results are presented in Figure~\ref{fig:atari-experiments}. Our method outperforms all baselines in 3 out of 4 tasks. We note that IRCR outperforms RRD in RoadRunner. It may be because IRCR is non-parametric and thus does not suffer from the difficulty of processing visual observations.
\section{Visualizing the Proxy Rewards of RRD}

\begin{figure}[h]
	\centering
	\includegraphics[width=0.98\linewidth]{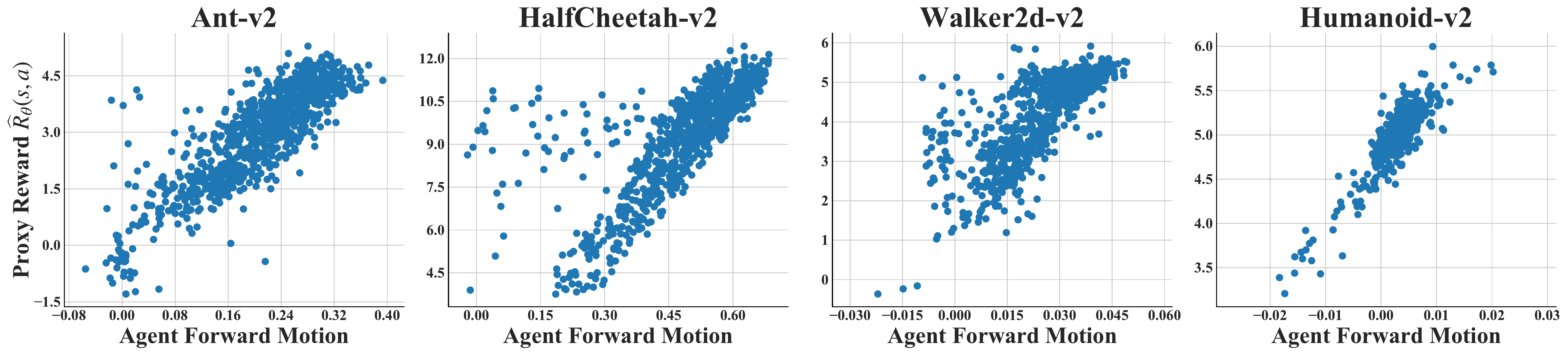}
	\caption{Visualization of the proxy rewards learned by RRD in MuJoCo locomotion tasks.}
	\label{fig:mujoco-visualization}
\end{figure}

In MuJoCo locomotion tasks, the goal of agents is running towards a fixed direction. In Figure~\ref{fig:mujoco-visualization}, we visualize the correlation between per-step forward distance and the assigned proxy reward. We uniformly collected $10^3$ samples during the first $10^6$ training steps. ``Agent Forward Motion'' denotes the forward distance at a single step. ``Proxy Reward $\widehat{R}_\theta(s,a)$'' denotes the immediate proxy reward assigned at that step. It shows that the learned proxy reward has high correlation to the forward distance at that step. 
\section{An Ablation Study on the Hyper-Parameters of IRCR}

We note that \citet{gangwani2020learning} uses a different hyper-parameter configuration from the standard SAC implementation \citep{haarnoja2018soft2}. The differences exist in two hyper-parameters:
\begin{table}[h]
	\centering
	\begin{tabular}{c c}
		\toprule
		Hyper-Parameter & Default Configuration \\
		\midrule
		Polyak-averaging coefficient & 0.001 \\
		\texttt{\#} transitions in replay buffer & $3\cdot 10^5$ \\
		\texttt{\#} transitions in mini-batch for training SAC & 512 \\
		\bottomrule
	\end{tabular}
	\caption{The hyper-parameters used by \citet{gangwani2020learning} in MuJoCo experiments.}
	\label{table:hyper-parameter-ircr}
\end{table}

To establish a rigorous comparison, we evaluate the performance of IRCR with the hyper-parameter configuration proposed by \citet{haarnoja2018soft2}, so that IRCR and RRD use the same hyper-parameters in their back-end SAC agents. The experiment results are presented in Figure~\ref{fig:ircr-ablation}.

\begin{figure}[h]
	\centering
	\includegraphics[width=0.98\linewidth]{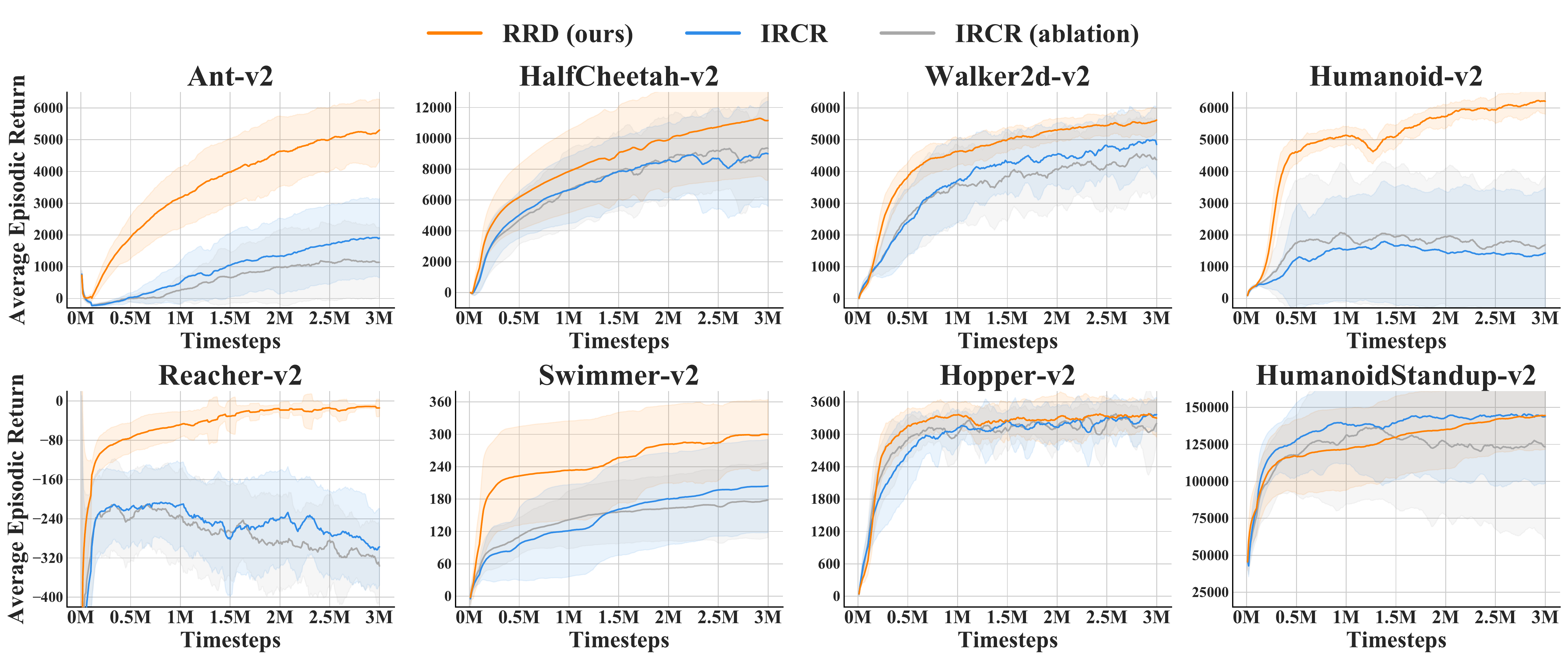}
	\caption{An ablation study on the hyper-parameter configuration of IRCR. The curves of ``IRCR'' refer to the performance of IRCR using the hyper-parameter setting proposed by \citet{gangwani2020learning}. The curves of ``IRCR (ablation)'' refer to the performance of IRCR using the hyper-parameters stated in Table \ref{table:hyper-parameter-mujoco}. All curves are plotted from 30 runs with random initializations.}
	\label{fig:ircr-ablation}
\end{figure}

As shown in Figure~\ref{fig:ircr-ablation}, the hyper-parameters tuned by \citet{gangwani2020learning} is more stable in most environments. Although using hyper-parameters stated in Table~\ref{table:hyper-parameter-mujoco} can improve the performance in some cases, the overall performance cannot outperform RRD.
\section{Related Work} \label{apx:related_work}

\paragraph{Delayed Feedback.}
Tackling environmental delays is a long-lasting problem in reinforcement learning and control theory \citep{nilsson1998stochastic, walsh2009learning, zhou2018distributed, zhou2019learning, heliou2020gradient, nath2021revisiting, tang2021bandit}. In real-world applications, almost all environmental signals have random delays \citep{schuitema2010control, hester2013texplore, amuru2014optimal, lei2020deep}, which is a fundamental challenge for the designs of RL algorithms. A classical method to handle delayed signals is stacking recent observations within a small sliding window as the input for decision-making \citep{katsikopoulos2003markov}. This simple transformation can establish a Markovian environment formulation, which is widely used to deal with short-term environmental delays \citep{mnih2015human}. Many recent works focus on establishing sample-efficient off-policy RL algorithm that is adaptive to delayed environmental signals \citep{bouteiller2021reinforcement, han2021off}. In this paper, we consider an extreme delay of reward signals, which is a harder problem setting than short-term random delays.

\paragraph{Reward Design.}
The ability of reinforcement learning agents highly depends on the designs of reward functions. It is widely observed that reward shaping can accelerate learning \citep{mataric1994reward, ng1999, devlin2011empirical, wu2017training, song2019playing}. Many previous works study how to automatically design auxiliary reward functions for efficient reinforcement learning. A famous paradigm, inverse reinforcement learning \citep{ng2000algorithms, fu2018learning}, considers to recover a reward function from expert demonstrations. Several works consider to learn a reward function from expert labels of preference comparisons \citep{wirth2016model, christiano2017deep, lee2021pebble}, which is a form of weak supervision. Another branch of work is learning an intrinsic reward function from experience that guides the agent to maximize extrinsic objective \citep{sorg2010reward, guo2016deep}. Such an intrinsic reward function can be learned through meta-gradients \citep{zheng2018learning, zheng2020can} or self-imitation \citep{guo2018generative, gangwani2019learning}. A recent work \citep{abel2021expressivity} studies the expressivity of Markov rewards and proposes algorithms to design Markov rewards for three notions of abstract tasks.

\paragraph{Temporal Credit Assignment.}
Another methodology for tackling long-horizon sequential decision problems is assigning credits to emphasize the contribution of each single step over the temporal structure. These methods directly consider the specification of the step values instead of manipulating the reward function. The simplest example is studying how the choice of discount factor $\gamma$ affects the policy learning \citep{petrik2008biasing, jiang2015dependence, fedus2019hyperbolic}. Several previous works consider to extend the $\lambda$-return mechanism \citep{sutton1988learning} to a more generalized credit assignment framework, such as adaptive $\lambda$ \citep{xu2018meta} and pairwise weights \citep{zheng2021pairwise}. RUDDER \citep{arjona2019rudder} proposes a return-equivalent formulation for the credit assignment problem and establish theoretical analyses \citep{holzleitner2021convergence}. Aligned-RUDDER \citep{patil2020align} considers to use expert demonstrations for higher sample efficiency. \citet{harutyunyan2019hindsight} opens up a new family of algorithms, called hindsight credit assignment, that attributes the credits from a backward view.

\paragraph{Value Decomposition.}
This paper follows the paradigm of reward redistribution that aims to decompose the return value to step-wise reward signals. The simplest mechanism in the literature is the uniform reward redistribution considered by \citet{gangwani2020learning}. It can be effectively integrated with off-policy reinforcement learning and thus achieves state-of-the-art performance in practice. Least-squares-based reward redistribution is investigated by \citet{efroni2021reinforcement} from a theoretical point of view. \citet{chatterji2021theory} extends the theoretic results to the logistic reward model. In game theory and multi-agent reinforcement learning, a related problem is how to attribute a global team reward to individual rewards \citep{nguyen2018credit, du2019liir, wang2020shapley}, which provide agents incentives to optimize the global social welfare \citep{vickrey1961counterspeculation, clarke1971multipart, groves1973incentives, myerson1981optimal}. A promising paradigm for multi-agent credit assignment is using structural value representation \citep{sunehag2018value, rashid2018qmix, son2019qtran, bohmer2020deep, wang2021towards, wang2021qplex}, which supports end-to-end temporal difference learning. This paradigm transforms the value decomposition to the structured prediction problem. A future work is integrating prior knowledge of the decomposition structure as many previous works for structured prediction \citep{chen2020rna, tavakoli2021learning}.

\end{document}